\documentclass[envcountsame,envcountsect,runningheads]{llncs}

\usepackage{times}
\usepackage{helvet}
\usepackage{courier}
\usepackage{amsmath}
\usepackage{amssymb}
\usepackage{upgreek}
\usepackage{algorithm}  
\usepackage{algorithmicx}
\usepackage[noend]{algpseudocode}
\usepackage{tikz}
\usetikzlibrary{arrows,automata}

\usepackage{marvosym}
\usepackage{fancybox,framed}
\usepackage{longtable}

\title{Incremental Satisfiability Modulo Theory for Verification of Deep Neural Networks}

\begin{document}

\author{ 
Pengfei Yang\inst{1} \and
Zhiming Chi\inst{1,2} \and
Zongxin Liu\inst{1,2} \and
Mengyu Zhao\inst{1,2} \and
Cheng-Chao Huang\inst{3} \and
Shaowei Cai\inst{1,2} \and
Lijun Zhang\inst{1,2}
}
%Anonymous
%Depeng Liu\inst{1,3} \and
%Lutan Zhao\inst{2,3} \and
%Pengfei Yang\inst{1,3} \and
%Bow-Yaw Wang \inst{4} \and
%Rui Hou\inst{2,3} \and
%Lijun Zhang\inst{1,3} \and
%Naijun Zhan\inst{1,3}
%}

\authorrunning{P. Yang, et al. } 
\institute{SKLCS, Institute of Software, Chinese Academy of Sciences, Beijing, China\and
University of Chinese Academy of Sciences, Beijing, China \and
Nanjing Institute of Software Technology, ISCAS, Nanjing, China}

\maketitle

\begin{abstract}
Constraint solving is an elementary way for verification of deep neural networks (DNN). In the domain of AI safety, a DNN might be modified in its structure and parameters for its repair or attack. For such situations, we propose the incremental DNN verification problem, which asks whether a safety property still holds after the DNN is modified. To solve the problem, we present an incremental satisfiability modulo theory (SMT) algorithm based on the Reluplex framework. We simulate the most important features of the configurations that infers the verification result of the searching branches in the old solving procedure (with respect to the original network), and heuristically check whether the proofs are still valid for the modified DNN. We implement our algorithm as an incremental solver called DeepInc, and exerimental results show that DeepInc is more efficient in most cases. For the cases that the property holds both before and after modification, the acceleration can be faster by several orders of magnitude,  showing that DeepInc is outstanding in incrementally searching for counterexamples.
 Moreover, based on the framework, we propose the  multi-objective DNN repair problem and give an algorithm based on our incremental SMT solving algorithm. Our repair method preserves more potential safety properties on the repaired DNNs compared with state-of-the-art.
\end{abstract}

\section{Introduction} \label{sec:introduction}
Deep  neural networks (DNN) have achieved exceptional performance in many fields like computer vision, natural language processing, game playing~\cite{alphago}, and malware detection.
However,  DNNs  are vulnerable to adversarial samples~\cite{SZSBEGF2014}, thus lack robustness. 
Even for a well-trained DNN, a very small (and even imperceptible) perturbation on the input may fool the network. 
This raises the concerns on the safety and reliability of DNNs when we deploy them in safety-critical applications like self-driving cars~\cite{selfdriving} and medical systems~\cite{medicalsystem}. 

Formal verification is a sound way to guarantee the robustness of DNNs.
In this work, we focus on safety properties of DNNs, i.e. given an input set of a DNN, we determine whether the output of the DNN is in a given safety region. 
The well-known local robustness of DNNs is also a safety property.
Constraint solving is an elementary and classical way for DNN verification.
The behaviour of the DNN and the negation of the property are encoded as constraints in the form of equations or inequalities of the involved real-valued variables and it is determined whether there exists an assignment for these variables such that the constraints hold. If so, the solver returns $\mathsf{SAT}$, and the assignment we find corresponds to a counterexample which violates the property, or otherwise it returns $\mathsf{UNSAT}$, and there does not exist an assignment that violates the property, so the property holds.

%In 2010, Luca Pulita et al. proposed the first method based on partition-refinement for DNN verification in~\cite{DBLP:conf/cav/PulinaT10}. 
In 2017, Katz et al.~\cite{reluplex} and Ehlers~\cite{planet} independently implemented Reluplex and Planet, two satisfiability module theory (SMT) solvers to verify DNNs with the $\mathrm{ReLU}$ activation function on properties expressible with SMT constraints, respectively. Reluplex encodes the ReLU activation function as a ReLU pair, and conducts a local search through pivoting for a given set of assertions. 
Differently, Planet uses the classical linear SMT to over-approximate the constraints and heuristically invokes constraint-driven  clause learning (CDCL) to cut branches in the search tree.
In 2020, as an optimized version of Reluplex, Marabou~\cite{marabou} showed a better experimental performance, although it still did not include a CDCL design. Except for SMT, a DNN verification problem can also be encoded as a mixed-integer linear programming (MILP), a semi-definite programming (SDP) or a linear programming (LP) problem.
Abstract interpretation based methods like AI${}^2$~\cite{AI2} and DeepPoly~\cite{deeppoly} are frequently used for the initialization of constraint solving, since they can efficiently offer an over-approximation of the semantics of the constraint solving problem.
%Since 2018, abstract interpretation has been one of the most popular methods for DNN verification in the lead of AI${}^2$~\cite{AI2}, and subsequent works like \cite{deepz,deeppoly,deepsymbol,charon,krelu,deeppolygpu,prodeep} have improved AI${}^2$ in terms of efficiency, precision and more activation functions (like sigmoid and $\tanh$) so that abstract interpretation based approach can be applied to DNNs of larger size and more complex structures. 

Incremental constraint solving is aimed to efficiently determine a satisfiablity problem when the constraints vary only a little bit.
The basic idea of constraint solving is that we refer to the old solving procedure for the original constraints, and consider whether the inferences in the old solving procedure can be inherited for the new constraints.
Since the changes in the constraints are rather small, most inferences are preserved with a quick check, and only a small number of branches which involve the changed constraints need a second calculation.
Incremental constraint solving has been widely used in formal methods, like incremental SAT solving in bounded model checking~\cite{DBLP:journals/entcs/EenS03,DBLP:conf/charme/Shtrichman01,DBLP:conf/fmics/SchrammelKBMTB15} and combinational equivalence checking~\cite{DBLP:conf/aspdac/DischS07}, sensitivity analysis for linear programming, etc.~\cite{DBLP:journals/networks/RaviW88,https://doi.org/10.1111/j.1540-5915.1991.tb00365.x,FILIPPI20051,DBLP:journals/eor/BorgonovoBW18,doi:10.1287/opre.43.6.948}

In the domain of AI safety, DNNs are often subject to such incremental situations. 
Constraints in DNN verification  consist of the property and the behaviour of DNNs, so small changes in the weights or even the structure of a DNN result in an incremental situation for DNN verification.  
The DNN repair problem, which aims to modify a DNN to fix its bad behaviors, usually results in a small change in its weights or even the structure. 
If we verify whether the repair is effective based on an old verification procedure, then it is natural to consider incremental constraint solving. 
Another application that may involve incremental constraint solving is the counterexample-guided abstraction refinement (CEGAR) framework of DNN verification, where we start with an abstraction of the DNN, and iteratively refine it if the verification returns a spurious counterexample. 
An refinement of the current DNN results in small changes in the structure and weights, and we naturally consider an incremental constraint solving for the verification of the refined DNN.
Other incremental situations for DNNs include adversarial training, backdoor attack, etc.

In this paper, we investigate algorithms for the incremental constraint solving problem for DNN verification.
Different from traditional incremental SAT solving or sensitivity analysis for linear programming, the changes of an incremental DNN verification problem  are often quantitatively (in the weights), but not qualitatively. 
That is, the number of changed constraints  might be large, and even all the constraints may be affected within an small perturbation, but the sum of the absolute values of the weight changes is limited. Thus, the incremental verification problem for DNN is substantially  different to the classical one.
Particularly, this brings us difficulties in borrowing techniques from  classical incremental SAT solving technique like re-using clauses and incremental solving with assumptions is not suitable for our setting.
%On the other hand, the fact that the changes are quantitatively small.

Based on the Reluplex framework, we propose the problem of incremental SMT solving for DNN verification. We consider the incremental situation that the DNN is only modified on its weights with its structure unchanged. In this situation, there is a natural one-to-one correspondence between the neurons of the original DNN and the modified one, and we can simulate the configurations that immediately infer the solving result in their branches in the modified DNN. In the simulation, we extract the most important features of these configurations, including the assertions of the branch, the set of basic variables, and the location of the linear equation that immediately infer an $\mathsf{UNSAT}$ result. By calculation on the conjunction of the assertions and the semantics of the modified DNN, we can easily reach the desired node in the search tree to simulate the assertions. Via a Gauss elimination, the tableau of linear equations can be transformed to the form with the same set of basic variables. The key linear equation that infers $\mathsf{UNSAT}$ can be located with the basic variable appearing in it. We notice that the key to successfully checking the old $\mathsf{UNSAT}$ proofs for the modified DNN is to tighten the numerical bounds of the variables in the key linear equation. Here we invoke linear programming in a standard way of encoding uncertain ReLU relations with a linear approximation~\cite{planet}. For the branch where a counterexample was found in the old solving procedure, we design some heuristic methods to search the corresponding branch and the branches near it for a counterexample of the modified DNN.
%\lz{a short paragraph describing the main technique of our algorithm is missing}

DNN repair is an incremental situation closely related to what we consider in incremental SMT solving. Based on our incremental techniques, we find that our incremental SMT solving is quite beneficial for DNN repair on safety properties, especially when we want to fix the violated properties and maintain a potential large set of properties at the same time -- which we refer to as the multi-objective DNN repair problem. Based on the repair method ART~\cite{art}, we add a restriction on the weight change rate in the repair so that most safety properties are highly likely to still hold. Our incremental SMT solving works for checking whether the violating property is successfully fixed in this algorithm.

The main contributions of this work are as follows:
\begin{itemize}
    \item We propose the problem of incremental constraint solving for DNN verification, and give an incremental SMT algorithm based on the Reluplex framework.
    \item We implement our algorithm as an incremental SMT solver DeepInc for DNN verification based on the SMT-based verifier Marabou. The experimental results show that our incremental constraint solving is more efficient in most cases when the modification to the DNNs is not large. Also, we analyse in detail the cases that incremental solving is less efficient, and all these cases show evidence that the modification, although small, has essentially changed the behaviour of the original DNN and leads to its SMT solving being more complicated.
    \item  We propose the problem of multi-objective DNN repair, and design an algorithm with our incremental techniques based on the repair method ART. The experimental results show that our method has advantages in maintaining more potential safety properties not violated, especially the local robustness properties.
\end{itemize}

\noindent\emph{Organisations of the paper.} We provide preliminaries in Section~\ref{sec:pre}.  We propose the incremental SMT solving problem for DNN verification, present the framework and the algorithm in Section~\ref{sec:main}. Section~\ref{sec:experiment} evaluates our algorithms through experiments. Based on the proposed incremental techniques, we present an algorithm for multi-objective DNN repair and compare it with state-of-the-art in Section~\ref{sec:repair}. Finally, we review related works and conclude the paper in Section~\ref{sec:relatedwork} and Section~\ref{sec:conclusion}, respectively.

\section{Preliminary} \label{sec:pre}

In this section, we recall some basic definitions on verification of deep neural networks. We write $a\wedge b:=\min(a,b)$ and $a\vee b:=\max(a,b)$ for $a,b \in \mathbb R$.

A deep neural network (DNN) is composed of a sequence of layers, starting with the input layer and ending with the output layer, with several hidden layers in between. 
In a layer, there are several neurons that each represent a real-valued variable. 
Except for neurons in the input layer, the value of a neuron is obtained by a function of the neurons in the previous layer.
We model a DNN as a function $f:\mathbb R^m \to \mathbb R^n$, which is the composition of the functions between layers. Such functions include affine functions and non-linear activation functions. 
An affine function is of the form $y=Wx+b$, where $W$ and $b$ are a constant real-valued matrix and vector called weight and bias, respectively. 
In this work, we only consider the ReLU activation function, defined as $\mathrm{ReLU}(x)=0 \vee x$ for $x \in \mathbb R$. For a ReLU relation $x_j=\mathrm{ReLU}(x_i)$, we say that the neuron $x_i$ is definitely activated/deactivated, if all the possible values of $x_i$ are non-negative/non-positive. If $x_i$ is neither definitely activated nor definitely deactivated, we say that it is uncertain.

\begin{figure}[t]
  \centering
  \scalebox{0.9}{
 \begin{tikzpicture}[->,>=stealth,auto,node distance=1.2cm,semithick,scale=1,every node/.style={scale=1}]
	\tikzstyle{blackdot}=[circle,fill=black,minimum size=6pt,inner sep=0pt]
	\tikzstyle{state}=[minimum size=0pt,circle,draw,thick]
	\tikzstyle{stateNframe}=[minimum size=0pt]	
	\node[state](x1){$x_1$};
	\node[state](x2)[below of=x1,yshift=-0.3cm]{$x_2$};
	\node[state](x3)[right of=x1,xshift=1.2cm]{$x_3$};
	\node[state](x4)[right of=x2,xshift=1.2cm]{$x_4$};
	\node[state](y1)[right of=x3,xshift=1.2cm]{$x_5$};
	\node[state](y2)[right of=x4,xshift=1.2cm]{$x_6$};
	\node[state](y)[right of=y1,xshift=1.2cm,yshift=-0.8cm]{$y$};
%	\node[blackdot](d1)[above of=r11,xshift=0.6cm]{};
\node[stateNframe](bias1)[below of=x3,yshift=0.7cm]{\tiny{bias $-0.1$}};

	\path (x1) edge	[-]						node {$0.2$} (x3)
		(x1) edge[-]							node[xshift=-0.6cm,yshift=0.15cm] {$0.8$} (x4)
		(x2) edge	[-]						node[xshift=-0.2cm,yshift=-0.35cm] {$-0.7$} (x3)
		(x2) edge[-]							node {$-0.8$} (x4)
			(x3) edge	[-]						node { $\mathrm{ReLU}(x_3)$} (y1)
		(x4) edge[-]							node {$\mathrm{ReLU}(x_4)$} (y2)
	(y1) edge[-]							node [xshift=-0.2cm]{$0.4$} (y)
			(y2) edge[-]							node[xshift=0.4cm] {$0.6$} (y)
			  ;
\end{tikzpicture}}
  \caption{A small neural network $f$} \label{fig:example1}
\end{figure}
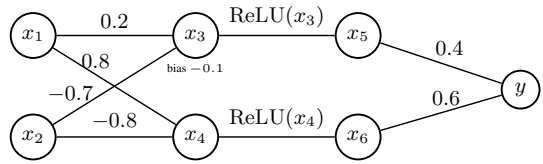

\begin{example}
In Fig.~\ref{fig:example1} we show a small neural network, where the input layer include two neurons $x_1$ and $x_2$ and it outputs  $y \in \mathbb R$. The behaviour of this neural network is the composition of an affine function $\begin{pmatrix}
x_3 \\
x_4
\end{pmatrix}= \begin{pmatrix}
0.2 & -0.7 \\
0.8 & -0.8
\end{pmatrix}\begin{pmatrix}
x_1 \\
x_2
\end{pmatrix}+\begin{pmatrix}
-0.1 \\
0
\end{pmatrix}$, ReLU activations $x_5=\mathrm{ReLU}(x_3)$ and $x_6=\mathrm{ReLU}(x_4)$, and an affine function $y=0.4x_5+0.6x_6$. If the input ranges in the box region $[-1,1] \times [-1,1]$, the neurons $x_3$ and $x_4$ are both uncertain.
\end{example}

In formal verification, a safety property means that nothing bad  happens. In the setting of DNN verification, a safety property typically requires that the DNN should behave correctly for all the inputs in a given region. Formally, we have the following definition.

\begin{definition}
A safety property of a DNN is a tuple $(f,X,P)$, where $f:\mathbb R^m \to \mathbb R^n$ is the neural network, and $X \subseteq \mathbb R^m$ and $P \subseteq \mathbb R^n$ are a set of input and output, respectively. The property $(f,X,P)$ holds iff for all $x \in X$, $f(x) \in P$. An input $x^* \in X$ which violates the property $P$, i.e., $f(x^*) \not \in P$, is called a counterexample of the property $(f,X,P)$.
A DNN verification problem is to determine whether a given property $(f,X,P)$ holds.
\end{definition}

%An $m$-dimensional polytope is a set of the form $\{x \in \mathbb R^m \mid Ax \le b\}$, where $A \in \mathbb R^{n \times m}$ and $b \in \mathbb R^{n}$. 
% Hereafter, we assume that the input region $X$ as well as the output region $P$ is always a polytope.
The well-known local robustness is also a safety property of DNNs. For a classification DNN $f:\mathbb R^m \to \mathbb R^n$, its output usually represents the score of each classification label ranging in $\{1,2,\ldots,n\}$, and the label with the largest score is selected as its classification result. Formally we define $C_f(x)=\arg \max_{1 \le i \le n} f(x)_i$ as the classification output of $f$. A classification DNN $f:\mathbb R^m \to \mathbb R^n$ is locally robust in a given neighborhood $B(x_0)$ of an input $x_0 \in \mathbb R^m$, if for all $x \in B(x_0)$, $C_f(x)=C_f(x_0)$. Such a local robustness property is a safety property $(f,X,P)$, where $X=B(x_0)$ and $P=\{y \in \mathbb R^n \mid \arg \max_i y_i=C_f(x_0)\}$. The most often used robustness region is the $L_\infty$-norm ball, defined as $B_\infty(x_0,r):=\{x \in \mathbb R^m \mid \|x-x_0\|_\infty\le r\}$ for $x_0 \in \mathbb R^m$ and $r>0$.

Constraint solving is an elementary way for DNN verification. Given a safety property $(f,X,P)$, we encode the input set $X$, the functions composing the DNN $f$, and the negation of the property $P$ as equations or inequalities, and determine whether the conjunction of these constraints is satisfiable. If it is satisfiable, there must exist a counterexample that violates the property, or else the property holds. %Since the property $P$ is an $n$-dimensional polytope, it negation $\neg P$ is a finite disjunction $\bigvee_i Q_i$ of polytopes, and we usually divide the problem according to this disjunction into the small ones in which every $Q_i$ is a polytope.

% Constraint solving is an elementary way for DNN verification. 
% Given a safety property $(f,X,P)$, we consider its negation $\neg (f,X,P)$ which can be encoded as a conjunction of equations and inequalities expressing the input region $X$, the output region $P$, and the functions composing the DNN $f$. Then we determine whether the negation is satisfiable. 
% If it is satisfiable, there must exist a counterexample that violates the safety property, or else the safety property holds. 
% Since the output region $P$ is an $n$-dimensional polytope, 
% its complement is a finite union $\bigcap_i Q_i$ of polytopes consisting of less constraints, by which we usually divide the problem into the small ones accordingly.

\begin{example} \label{example:2}
We consider the DNN $f$ in Fig.~\ref{fig:example1}, with the input region $X=[-1,1] \times [-1,1]$ and the property $P=\{y \mid y<0.3\}$. The encoding for the property $(f,X,P)$ for constraint solving is the conjunction of the  input constraints 
$
-1 \le x_1 \le 1, -1 \le x_2 \le 1,
$,
the negation of the property
$
y \ge 0.3,
$,
and the behaviours of the DNN $f$, i.e.,
$x_3 = 0.2x_1 - 0.7 x_2 - 0.1, 
x_4 = 0.8x_1 - 0.8 x_2, 
x_5 = \mathrm{ReLU}(x_3), 
x_6 = \mathrm{ReLU}(x_4), 
y = 0.4x_5 + 0.6x_6$.
\end{example}

\section{Incremental Constraint Solving for DNN Verification} \label{sec:main}

In this section, we introduce the incremental SMT solving problem for DNN verification, and present our main algorithm.

\subsection{Reluplex}
Reluplex~\cite{reluplex} is an SMT-based DNN verifier. 
We recall first Reluplex, adapted to our notation. 
In a Reluplex solving, a configuration of the current constraints is managed. When the configuration is not finished, it stores the linear equations, the ReLU relations, the numerical range and the current assignment of all the variables. 
\begin{definition}
A Reluplex configuration $\mathcal C$ on a set $\mathcal X$ of variables is $\mathsf{SAT}$, $\mathsf{UNSAT}$, or a tuple $(\mathcal B,T,R,l,u,\alpha)$, where
\begin{itemize}
    \item $\mathcal B \subseteq \mathcal X$ is the set of basic variables,
    \item $T$ is the tableau which stores constraints of the form $x_i=\sum_{x_j \not \in \mathcal B} a_{ij}x_j$ with $x_i \in \mathcal B$,
    \item $R \subseteq \mathcal X \times \mathcal X$ is the set of ReLU pairs,
    \item $l,u:\mathcal X \to \mathbb R$ is the lower and the upper bound of each variable, respectively,
    \item and $\alpha:\mathcal X \to \mathbb R$ is the assignment of each variable. 
\end{itemize}
\end{definition}
Here for a non-terminating configuration, Reluplex follows the idea of the algorithm Simplex that variables are divided into two categories, basic variables and non-basic variables, with all the linear equations of the form that a basic variable equals a linear combination of non-basic variables. When we search for an assignment to satisfy all the constraints, we can straight change the value of non-basic variables. When we reach a step where the value of a basic variable needs to be changed, we cannot straight change it. Instead, we should first conduct a pivot operation, i.e., replacing a basic variable with a non-basic variable. When we want to adjust the value of a basic variable $x_i$, we locate ourselves in the linear equation $x_i=\sum_{x_j \not \in \mathcal B} a_{ij}x_j$, choose a non-basic variable $x_k$ which appears in this equation, rewrite the linear equation as 
$$
x_k=\frac{1}{a_{ik}} x_i-\sum_{x_j \not \in \mathcal B\cup\{x_k\}} \frac{a_{ij}}{a_{ik}} x_j,
$$
and substitute $x_k$ in the other equations with the right hand side of this equation. Now the variable $x_i$ has become a non-basic variable, and we can change its value.

Given a property $(f,X,P)$, we initialize the configuration by introducing slack variables to establish the tableau and invoking DeepPoly to compute the numerical bound. 
When we reach a figuration $(\mathcal B,T,R,l,u,\alpha)$, a local search is conducted by pivoting the variables and adjusting the assigment $\alpha$ of variables to match the constraints given by the tableau $T$ and the ReLU relation $R$. 
If the assignment in the current figuration $(\mathcal B,T,R,l,u,\alpha)$ makes all the constraints satisfied, i.e.,
\begin{itemize}
    \item for any linear equation $x_i=\sum_{x_j \not \in \mathcal B} a_{ij}x_j$ in the tableau $T$, it holds that $\alpha(x_i)=\sum_{x_j \not \in \mathcal B} a_{ij}\alpha(x_j)$,
    \item for any pair $(x_i,x_j) \in R$, it holds that $\alpha(x_j)=\mathrm{ReLU}(\alpha(x_i))$,
    \item and for any $x \in \mathcal X$, $l(x) \le \alpha(x) \le u(x)$,
\end{itemize}
then a counterexample is found and Reluplex outputs $\mathsf{SAT}$. 
If  there is a linear equation $x_i=\sum_{x_j \not \in \mathcal B} a_{ij}x_j$ in the tableau $T$ satisfying 
\begin{align} \label{eq:unsat1}
l(x_i)>\sum_{x_j \not \in \mathcal B} (a_{ij} \vee 0) \cdot u(x_j)+\sum_{x_j \not \in \mathcal B} (a_{ij} \wedge 0) \cdot l(x_j)
\end{align}
or
\begin{align}\label{eq:unsat2}
u(x_i)<\sum_{x_j \not \in \mathcal B} (a_{ij} \vee 0) \cdot l(x_j)+\sum_{x_j \not \in \mathcal B} (a_{ij} \wedge 0) \cdot u(x_j),
\end{align}
then no assignment can satisfy the constraints in this configuration, and we mark $\mathsf{UNSAT}$ for the branch of this configuration.
A local search may not determine $\mathsf{SAT}$ or $\mathsf{UNSAT}$ for a configuration. After a certain number of steps, we stop the local search and choose an uncertain neuron $x_j$ to split its numerical bound into $[0,u(x_j)]$ and $[l(x_j),0]$. We assert the constraint $x_j \ge 0$ or $x_j \le 0$ to the current configuration and conduct a new local search. If all the branches are marked with $\mathsf{UNSAT}$, then Reluplex outputs $\mathsf{UNSAT}$. The following example shows how the configuration is managed in a Reluplex solving.

\begin{figure}[t]
  \centering
  \scalebox{0.9}{
 \begin{tikzpicture}[->,>=stealth,auto,node distance=1.2cm,semithick,scale=1,every node/.style={scale=1}]
	\tikzstyle{blackdot}=[circle,fill=black,minimum size=6pt,inner sep=0pt]
	\tikzstyle{state}=[minimum size=0pt,circle,draw,thick]
	\tikzstyle{stateNframe}=[minimum size=0pt]	
	\node[state](x1){$x_1$};
	\node[state](x2)[below of=x1,yshift=-0.3cm]{$x_2$};
	\node[state](x3)[right of=x1,xshift=2cm]{$x_3$};
	\node[state](x4)[right of=x2,xshift=2cm]{$x_4$};
	\node[state](y1)[right of=x3,xshift=2.3cm]{$x_5$};
	\node[state](y2)[right of=x4,xshift=2.3cm]{$x_6$};
	\node[state](y)[right of=y1,xshift=1.5cm,yshift=-0.8cm]{$y$};
%	\node[blackdot](d1)[above of=r11,xshift=0.6cm]{};

\node[stateNframe](f1)[above of=x1,yshift=-0.5cm]{$u_1=1$};	
\node[stateNframe](f2)[above of=f1,yshift=-0.85cm]{$l_1=-1$};	
\node[stateNframe](f3)[above of=f2,yshift=-0.85cm]{$x_1\le 1$};	
\node[stateNframe](f4)[above of=f3,yshift=-0.85cm]{$x_1 \ge -1$};

\node[stateNframe](g1)[above of=x3,yshift=-0.5cm]{$u_3=0.8$};	
\node[stateNframe](g2)[above of=g1,yshift=-0.85cm]{$l_3=-1$};	
\node[stateNframe](g3)[above of=g2,yshift=-0.85cm]{$x_3\le 0.2x_1-0.7x_2-0.1$};	
\node[stateNframe](g4)[above of=g3,yshift=-0.85cm]{$x_3 \ge 0.2x_1-0.7x_2-0.1$};

\node[stateNframe](h1)[above of=y1,yshift=-0.5cm]{$u_5=0.8$};	
\node[stateNframe](h2)[above of=h1,yshift=-0.85cm]{$l_5=0$};	
\node[stateNframe](h3)[above of=h2,yshift=-0.85cm]{$x_5\le 0.445 x_3+0.445$};	
\node[stateNframe](h4)[above of=h3,yshift=-0.85cm]{$x_5\ge 0$};

\node[stateNframe](i1)[below of=x2,yshift=0.5cm]{$x_2\ge -1$};	
\node[stateNframe](i2)[below of=i1,yshift=0.85cm]{$x_2 \le 1$};	
\node[stateNframe](i3)[below of=i2,yshift=0.85cm]{$l_2= -1$};	
\node[stateNframe](i4)[below of=i3,yshift=0.85cm]{$u_2 = 1$};

\node[stateNframe](j1)[below of=x4,yshift=0.5cm]{$x_4\ge 0.8x_1-0.8x_2$};	
\node[stateNframe](j2)[below of=j1,yshift=0.85cm]{$x_4 \le 0.8x_1-0.8x_2$};
\node[stateNframe](j3)[below of=j2,yshift=0.85cm]{$l_4= -1.6$};	
\node[stateNframe](j4)[below of=j3,yshift=0.85cm]{$u_4 = 1.6$};

\node[stateNframe](k1)[below of=y2,yshift=0.5cm]{$x_6\ge 0$};	
\node[stateNframe](k2)[below of=k1,yshift=0.85cm]{$x_6 \le 0.5x_4+0.8$};
\node[stateNframe](k3)[below of=k2,yshift=0.85cm]{$l_6= 0$};	
\node[stateNframe](k4)[below of=k3,yshift=0.85cm]{$u_6 = 1.6$};

\node[stateNframe](l1)[above of=y,yshift=-0.5cm,xshift=1cm]{$u_7=1.28$};	
\node[stateNframe](l2)[above of=l1,yshift=-0.85cm]{$l_7=0$};	
\node[stateNframe](l3)[above of=l2,yshift=-0.85cm]{$y\le 0.4 x_5+0.6x_6$};	
\node[stateNframe](l4)[above of=l3,yshift=-0.85cm]{$y\ge 0.4 x_5+0.6x_6$};
\node[stateNframe](bias1)[below of=x3,yshift=0.7cm]{\tiny{bias $-0.1$}};

%\node[stateNframe](a)[below of=j4,yshift=0.5cm]{(a) The original DNN $f$};

	\path (x1) edge	[-]						node {$0.2$} (x3)
		(x1) edge[-]							node[xshift=-0.6cm,yshift=0.15cm] {$0.8$} (x4)
		(x2) edge	[-]						node[xshift=-0.2cm,yshift=-0.35cm] {$-0.7$} (x3)
		(x2) edge[-]							node {$-0.8$} (x4)
			(x3) edge	[-]						node { $\mathrm{ReLU}(x_3)$} (y1)
		(x4) edge[-]							node {$\mathrm{ReLU}(x_4)$} (y2)
	(y1) edge[-]							node [xshift=-0.2cm]{$0.4$} (y)
			(y2) edge[-]							node[xshift=0.4cm] {$0.6$} (y)
			  ;
\end{tikzpicture}}
  \caption{The DeepPoly abstraction of the DNN $f$} \label{fig:deeppolyresult}
\end{figure}
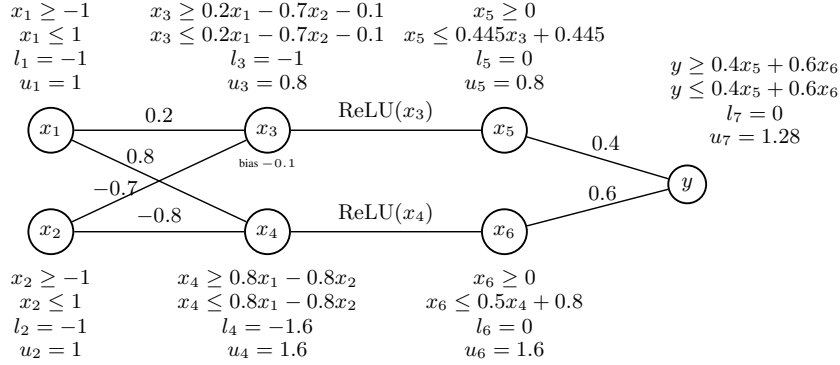

\begin{example} \label{example:marabou}
We consider the safety property $(f,X,P)$ in Example~\ref{example:2}. We first initialize its Reluplex configuration. The DeepPoly abstraction of the DNN $f$ on $X$ is shown in Fig.~\ref{fig:deeppolyresult}. The set of ReLU reltions is $R=\{(x_3,x_5),(x_4,x_6)\}$. The initial tableau is constructed from the affine functions in $f$, i.e., $x_3 = 0.2x_1 - 0.7 x_2 - 0.1$, $x_4 = 0.8x_1 - 0.8 x_2$, and $y = 0.4x_5 + 0.6x_6$, and the inequalities from the ReLU relations, i.e., $x_5 \ge x_3$ and $x_6 \ge x_4$. We first introduce two slack variables $x_7$ and $x_8$ for the two inequalities as
    $x_5 - x_3 - x_7 = 0$ and
    $x_6-x_4-x_8=0$.
From the DeepPoly abstraction we know that the numerical bounds of $x_7$ and $x_8$ are $[0,1]$ and $[0,1.6]$, respectively. For each of the five linear equations here, we add a slack variable to represent the constants, i.e.,
\begin{align*}
x_9 &= -x_3+0.2x_1-0.7x_2, \\
x_{10} &= 0.8x_1-0.8x_2 -x_4, \\
x_{11} &= 0.4x_5+0.6x_6-y, \\
x_{12} &=  x_5-x_3-x_7,                     \\
x_{13} &=  x_6-x_4-x_8.               
\end{align*}
with $x_9=0.1$ and $x_{10},x_{11},x_{12},x_{13}=0$. For each linear equation above, we heuristically choose a variable as the basic variable, and we obtain the initial set of basic variables $\mathcal B_0=\{x_3,x_4,y,x_7,x_8\}$, the initial tableau $T_0$:
\begin{align*}
x_3 &= 0.2 x_1 - 0.7x_2 - x_9, \\
x_4 &= 0.8x_1-0.8x_2 -x_{10}, \\
y &= 0.4x_5+0.6x_6 -x_{11}, \\
x_7 &= {\color{gray} x_5 -x_3 -x_{12} =} -0.2 x_1 + 0.7x_2 + x_5 + x_9 -x_{12}, \\
x_8 &= {\color{gray} x_6-x_4-x_{13} =} -0.8x_1 + 0.8 x_2 + x_6 + x_{10} - x_{13},
\end{align*}
and the initial numerical bounds and assignments:
\begin{table}[h]
\begin{tabular}{ccccccccccccccc}
\hline
 & $x_1$ & $x_2$ & $x_3$ & $x_4$ & $x_5$ & $x_6$ & $y$ & $x_7$ &$x_8$  & $x_9$ & $x_{10}$ & $x_{11}$ & $x_{12}$ & $x_{13}$ \\ \hline
$l$ & $-1$  & $-1$ & $-1$ & $-1.6$ & $0$ & $0$ & $0.3$ & $0$ & $0$ & $0.1$ & $0$ & $0$ & $0$ &  $0$\\ \hline
$u$ & $1$ & $1$ & $0.8$ & $1.6$ & $0.8$ & $1.6$ & $1.28$ & $1$ & $1.6$ & $0.1$ & $0$ & $0$ & $0$ & $0$ \\ \hline
$\alpha$ & $-1$ & $-1$ & $0.4$ & $0$ & $0$ & $0$ & $\mathbf{0}$ & $\mathbf{-0.4}$ & $0$ & $0.1$ & $0$ & $0$ & $0$ & $0$ \\ \hline
\end{tabular}
\end{table}

Notice that the assignment of $y$ and $x_4$ does not satisfy the constraints of their numerical bounds, and that they are both currently basic variables, so we choose one of them, for example $x_7$, to pivot. The current assignment of $x_7$ is smaller than its lower bound, and changes on the assignment of $x_5$ may repair this conflict, because $x_5$ appears positively and it has not reached its upper bound. We have $x_5=0.2x_1-0.7x_2+x_7-x_9+x_{12}$ and we substitute the appearance of $x_5$ in the other equations with $0.2x_1-0.7x_2+x_7-x_9+x_{12}$, so now the set of basic variables is $\mathcal B_1=\{x_3,x_4,x_5,y,x_8\}$ and the tableau becomes $T_1$:
\begin{align*}
x_3 &= 0.2 x_1 - 0.7x_2 - x_9, \\
x_4 &= 0.8x_1-0.8x_2 -x_{10}, \\
x_5 &=0.2x_1-0.7x_2+x_7-x_9+x_{12}, \\
y &= 0.08x_1-0.28x_2+0.6x_6+0.4x_7-0.4x_9-x_{11}+0.4x_{12}, \\
x_8 &= -0.8x_1 + 0.8 x_2 + x_6 + x_{10} - x_{13}.
\end{align*}
Now $x_7$ is a non-basic variable, so we change its assignment to $0$ to satisfy its numerical bounds, and at the same time the assignment of the basic variables $x_5$ and $y$ changes according to $T_1$. We show the current assignment in the following table.
\begin{table}[h]
\begin{tabular}{ccccccccccccccc}
\hline
 & $x_1$ & $x_2$ & $x_3$ & $x_4$ & $x_5$ & $x_6$ & $y$ & $x_7$ &$x_8$  & $x_9$ & $x_{10}$ & $x_{11}$ & $x_{12}$ & $x_{13}$ \\ \hline
$l$ & $-1$  & $-1$ & $-1$ & $-1.6$ & $0$ & $0$ & $0.3$ & $0$ & $0$ & $0.1$ & $0$ & $0$ & $0$ &  $0$\\ \hline
$u$ & $1$ & $1$ & $0.8$ & $1.6$ & $0.8$ & $1.6$ & $1.28$ & $1$ & $1.6$ & $0.1$ & $0$ & $0$ & $0$ & $0$ \\ \hline
$\alpha$ & $-1$ & $-1$ & $0.4$ & $0$ & $0.4$ & $0$ & $\mathbf{0.16}$ & $0$ & $0$ & $0.1$ & $0$ & $0$ & $0$ & $0$ \\ \hline
\end{tabular}
\end{table}

Now the assignment of $y$ still violates its numerical bounds, so we pivot $y$ with the non-basic variable $x_6$, and adjust its assignment to $0.3$. Following a similar procedure, we obtain the set of basic variables $\mathcal B_2=\{x_3,x_4,x_5,x_6,x_8\}$, the tableau $T_2$ and the assignment:
\begin{align*}
x_3 &= 0.2 x_1 - 0.7x_2 - x_9, \\
x_4 &= 0.8x_1-0.8x_2 -x_{10}, \\
x_5 &=0.2x_1-0.7x_2+x_7-x_9+x_{12}, \\
x_6 &=- 0.13x_1+0.47x_2+1.67y-0.67x_7+0.4x_7-0.67x_9+1.67x_{11}-0.67x_{12}, \\
x_8 &= -0.93 x_1+1.27x_2+1.67y-0.67x_7+0.67x_9+x_{10}+1.67x_{11}-0.67x_{12}- x_{13},
\end{align*}

\begin{table}[h!]
\begin{tabular}{ccccccccccccccc}
\hline
 & $x_1$ & $x_2$ & $x_3$ & $x_4$ & $x_5$ & $x_6$ & $y$ & $x_7$ &$x_8$  & $x_9$ & $x_{10}$ & $x_{11}$ & $x_{12}$ & $x_{13}$ \\ \hline
$l$ & $-1$  & $-1$ & $-1$ & $-1.6$ & $0$ & $0$ & $0.3$ & $0$ & $0$ & $0.1$ & $0$ & $0$ & $0$ &  $0$\\ \hline
$u$ & $1$ & $1$ & $0.8$ & $1.6$ & $0.8$ & $1.6$ & $1.28$ & $1$ & $1.6$ & $0.1$ & $0$ & $0$ & $0$ & $0$ \\ \hline
$\alpha$ & $-1$ & $-1$ & $0.4$ & $0$ & $0.4$ & $0.233$ & $0.3$ & $0$ & $0.233$ & $0.1$ & $0$ & $0$ & $0$ & $0$ \\ \hline
\end{tabular}
\end{table}

This assignment does not violate any numerical bound, and all the linear equations in $T_2$ are satisfied, but the ReLU relation $x_6=\mathrm{ReLU}(x_4)$ is violated. If in this step we pivot $x_6$ with $y$, we again obtain $\mathcal B_1$, $T_1$ and the same conflict assignment as last step. Assume that now the threshold of local search is reached, so we must choose an uncertain ReLU neuron to split. Here a heuristic choice is to split $x_4$, since there has been one conflict in the local search for $x_4$ but none for $x_3$. For the deactivated branches where the assertion is $x_4 \le 0$, we first reset the upper bound of $x_4$ to be $0$, and we again run DeepPoly additionally under this assertion to obtain new numerical bounds and assignment:
\begin{table}[h]
\begin{tabular}{ccccccccccccccc}
\hline
 & $x_1$ & $x_2$ & $x_3$ & $x_4$ & $x_5$ & $x_6$ & $y$ & $x_7$ &$x_8$  & $x_9$ & $x_{10}$ & $x_{11}$ & $x_{12}$ & $x_{13}$ \\ \hline
$l$ & $-1$  & $-1$ & $-1$ & $-1.6$ & $0$ & $0$ & $0.3$ & $0$ & $0$ & $0.1$ & $0$ & $0$ & $0$ &  $0$\\ \hline
$u$ & $1$ & $1$ & $0.8$ & $0$ & $0.8$ & $0$ & $0.32$ & $1$ & $1.6$ & $0.1$ & $0$ & $0$ & $0$ & $0$ \\ \hline
$\alpha$ & $-1$ & $-1$ & $0.4$ & $0$ & $0.4$ & $0$ & $\mathbf{0.16}$ & $0$ & $0$ & $0.1$ & $0$ & $0$ & $0$ & $0$ \\ \hline
\end{tabular}
\end{table}

Now the variable $y$ again violates its numerical bounds, and we pivot it with a non-basic variable. After certain steps of local search, we cannot find an assignment to satisfy all the constraints, so we have to again split, and the assertion is $x_3 \le 0$ or $x_3 \ge 0$. In both of the branches, the problem is reduced to an LP problem, and both branches here are infeasible.

We still have the branch $x_4 \ge 0$ not solved. We follow a similar way of DeepPoly restricted on the assertion $x_4 \ge 0$, checking whether the assignment is valid, and pivot a basic variable which violates the constraints. After certain steps of local search, again we need to make split on $x_3$. On the branch of $x_3 \le 0$, we find an assignment that satisfies all the constraints, in which
%\begin{table}[h]
%\begin{tabular}{ccccccccccccccc}
%\hline
% & $x_1$ & $x_2$ & $x_3$ & $x_4$ & $x_5$ & $x_6$ & $y$ & $x_7$ &$x_8$  & $x_9$ & $x_{10}$ & $x_{11}$ & $x_{12}$ & $x_{13}$ \\ \hline
%$l$ & $-1$  & $0.05$ & $-0.475$ & $0.5$ & $0$ & $0$ & $0.3$ & $0$ & $0$ & $0.1$ & $0$ & $0$ & $0$ &  $0$\\ \hline
%$u$ & $1$ & $1$ & $0$ & $1.6$ & $0$ & $1.6$ & $0.96$ & $1$ & $0$ & $0.1$ & $0$ & $0$ & $0$ & $0$ \\ \hline
%$\alpha$ & $0.675$ & $0.05$ & $0$ & $0.5$ & $0$ & $0.5$ & $0.3$ & $0$ & $0$ & $0.1$ & $0$ & $0$ & $0$ & $0$ \\ \hline
%\end{tabular}
%\end{table}
the counterexample is $(0.675,0.05)^\mathrm{T}$ whose output is $y=0.3$, indeed violating the property $P:y<0.3$. Since we have found a true counterexample, the last branch where $x_3 \ge 0$ will not be solved, and Reluplex immediately outputs $\mathsf{SAT}$.%\footnote{All the results reported in this example has been checked by the solver Marabou.}
\end{example}

\subsection{Incremental SMT Solving Problem}
As can be seen, a Reluplex solving is actually conducting a depth-first search according to the assertion which limits the behaviour of an uncertain ReLU neuron as definitely activated or deactivated, so we can formalize a Reluplex solving procedure as a labelled binary tree.
\begin{definition}
A Reluplex solving procedure is a labelled binary tree $\mathcal T=(V,E,r,L)$, where 
\begin{itemize}
\item $V$ is the set of nodes, 
\item $E \subseteq V \times V$ is the set of edges, 
\item $r \in V$ is the root, 
\item and $L$ is a labelling function which assigns each node $v \in V$ to a finite sequence of configurations and each edge $e \in E$ to an assertion of the form $x_j \ge 0$ or $x_j \le 0$, where $x_j$ is an uncertain neuron. 
\end{itemize}
\end{definition}
If the output of a Reluplex solving is $\mathsf{UNSAT}$, all the leaves in this tree have label of the form $\mathcal C_1, \ldots, \mathcal C_n,\mathsf{UNSAT}$, and we call them UNSAT leaves. If the output is $\mathsf{SAT}$, there exists a leaf whose label is of the form $\mathcal C_1, \ldots, \mathcal C_n,\mathsf{SAT}$, which we call a SAT leaf, and the other leaves are either $\mathsf{UNSAT}$ leaves or labelled with the empty sequence $\varepsilon$. For a non-empty sequence $\ell$ of configurations, we use $\ell \downarrow$ to represent the last configuration that is not $\mathsf{SAT}$ or $\mathsf{UNSAT}$ in $\ell$. For a node $v \in V$, we use $\mathbf{Assert}(v)$ to denote the set of the assertions labelling the edges from the root $r$ to the node $v$. Notice that the semantics of a node $v$ is the intersection of the initial configuration and the conjunction of the assertions in $\mathbf{Assert}(v)$, so we also write this constraint $\bigwedge_{Q \in \mathbf{Assert}(v)} Q$ as $\mathbf{Assert}(v)$ for simplicity.

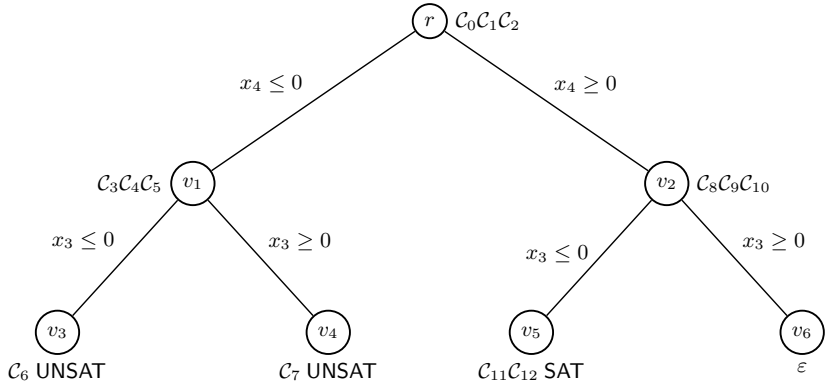
\begin{figure}[t]
  \centering
  \scalebox{0.9}{
 \begin{tikzpicture}[->,>=stealth,auto,node distance=.6cm,semithick,scale=.8,every node/.style={scale=1}]
	\tikzstyle{blackdot}=[circle,fill=black,minimum size=6pt,inner sep=0pt]
	\tikzstyle{state}=[minimum size=0pt,circle,draw,thick]
	\tikzstyle{stateNframe}=[minimum size=0pt]	
	\node[state](root)[label=right:$\mathcal C_0\mathcal C_1\mathcal C_2$]{$r$};
	\node[state](11)[below of=root,xshift=-3.5cm,yshift=-1.8cm, label=left:$\mathcal C_3 \mathcal C_4 \mathcal C_5$]{$v_1$};
	\node[state](12)[below of=root,xshift=3.5cm,yshift=-1.8cm,label=right:$\mathcal C_8\mathcal C_9\mathcal C_{10}$]{$v_2$};
	\node[state](21)[below of=11,xshift=-2cm,yshift=-1.6cm,label=below:$\mathcal{ C}_6\ \mathsf{UNSAT}$]{$v_3$};
	\node[state](22)[below of=11,xshift=2cm,yshift=-1.6cm, label=below:$\mathcal C_7\ \mathsf{UNSAT}$]{$v_4$};
	\node[state](23)[below of=12,xshift=-2cm,yshift=-1.6cm, label=below:$\mathcal C_{11} \mathcal C_{12}\ \mathsf{SAT}$]{$v_5$};
	\node[state](24)[below of=12,xshift=2cm,yshift=-1.6cm,label=below:$\varepsilon$]{$v_6$};
%	\node[blackdot](d1)[above of=r11,xshift=0.6cm]{};

%\node[stateNframe](f1)[above of=x1,yshift=-0.5cm]{$u_1=1$};	

	\path (root) edge	[-]						node[xshift=-1.2cm,yshift=0.5cm] {$x_4 \le 0$} (11)
		(root) edge[-]							node {$x_4 \ge 0$} (12)
		(11) edge	[-]						node[xshift=-1.2cm,yshift=0.5cm] {$x_3 \le 0$} (21)
		(11) edge[-]							node {$x_3 \ge 0$} (22)
			(12) edge	[-]						node[xshift=-1.2cm,yshift=0.3cm] { $x_3 \le 0$} (23)
		(12) edge[-]							node {$x_3 \ge 0$} (24)
			  ;
\end{tikzpicture}}
  \caption{The solving procedure of $(f,X,P)$ in  Example~\ref{example:marabou}} \label{fig:oldtree}
\end{figure}

\begin{example}
We revisit Example~\ref{example:marabou}. Its solving procedure $\mathcal T=(V,E,r,L)$ is shown in Fig.~\ref{fig:oldtree}, where
\begin{itemize}
    \item $V=\{r,v_1,v_2,v_3,v_4,v_5,v_6\}$ with $r$ the root,
    \item $E=\{(r,v_1),(r,v_2),(v_1,v_3),(v_1,v_4),(v_2,v_5),(v_2,v_6)\}$,
    \item and the labelling function $L$ on $V \cup E$ is marked in Fig.~\ref{fig:oldtree}.
\end{itemize}
In this solving procedure, the leaves $v_3$ and $v_4$ are UNSAT leaves, $v_5$ is the SAT leaf, and $v_6$ is a leaf labelled with the empty sequence $\varepsilon$. 
By solving the node $v_5$, we have found a true counterexample, so the last leaf $v_6$ is not solved in the solving procedure, and its label should be $\varepsilon$. The configuration $\mathcal C_0$ is the initial configuration, and $\mathcal C_1,\mathcal C_2,\mathcal C_3$ are the first three configurations we obtain in the solving procedure, which have been presented in detail in Example~\ref{example:marabou}.
\end{example}

Incremental constraint solving for DNN verification is aimed to efficiently verify a property on the modified DNN $f'$ by making use of the verification procedure of the original DNN $f$. In this work, we consider incremental SMT solving based on the Reluplex framework, and assume that the modified DNN $f'$ has exactly the same structure as the original DNN $f$, and only slightly differs in the weights and bias in the affine functions. We formally state our incremental SMT solving problem for DNN verification as follows:

\begin{quote}\emph{\textbf{Incremental SMT Solving for DNN Verification.}}
    Given the Reluplex solving procedure $\mathcal T=(V,E,r,L)$ of a safety property $(f,X,P)$, we determine whether the property $(f',X,P)$ holds, where the DNN $f'$ only differs from $f$ in the weights and bias in the affine functions.
\end{quote}

For this incremental SMT solving problem, the most significant intuition is that, the modified DNN $f'$ is slightly different from the original DNN $f$, so it is highly possible that the verification procedure of $f'$ inherits that of $f$ in most branches. 
This motivates us to directly locate the configurations in the leaves, especially the one before the last which immediately infers $\mathsf{UNSAT}$ or $\mathsf{SAT}$ in the original solving. 
We check whether such inference still works for the modified DNN $f'$. 

\begin{figure}[t]
    \centering
\includegraphics[width=0.9\linewidth]{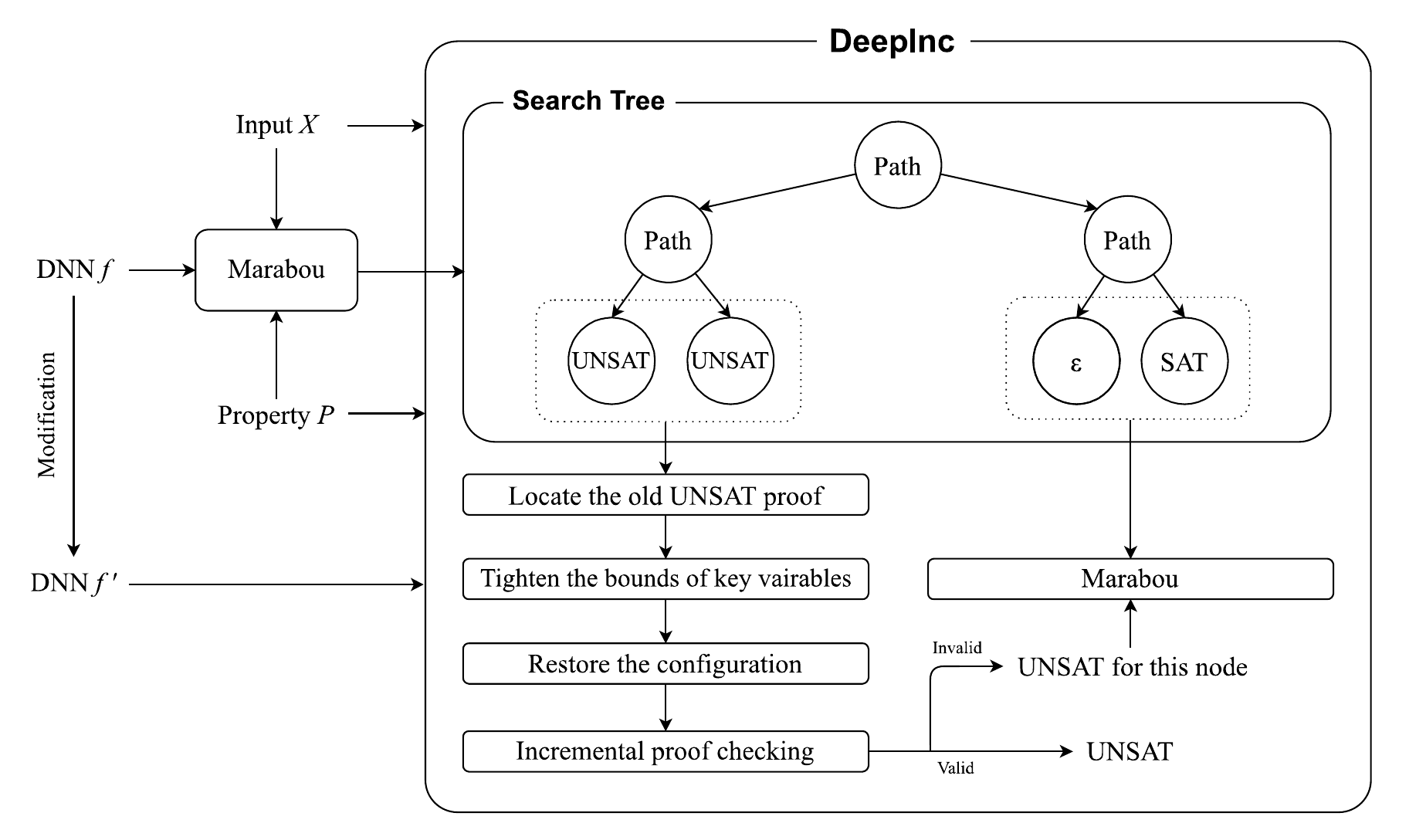}
    \caption{The framework of incremental SMT solving for DNN verification} \label{fig:framework}
\end{figure}

For a leaf node $v$ in $\mathcal T$, we try to extract the key features in the configuration $L(v) \downarrow$, and add these features to the initial configuration $\mathcal C_0$ for the modified DNN $f'$. The first feature is naturally $\mathbf{Assert}(v)$, the conjunction of the assertions on the edges in the path from $r$ to $v$. By adding $\mathbf{Assert}(v)$ to the initial configuration $\mathcal C_0$, we preserve all the behaviour of the uncertain neurons involved in $\mathbf{Assert}(v)$. Another important feature is the set of basic variables in the configuration $L(v) \downarrow$, especially when $v$ is an UNSAT leaf. Recall the inference rule for $\mathsf{UNSAT}$ (Equ.~\eqref{eq:unsat1} and \eqref{eq:unsat2}), and we find that the inference of $\mathsf{UNSAT}$ heavily relies on the linear equation, denoted by $(L(v) \downarrow)^*$ that implies a contradiction on the numerical bound of the variables in it. This linear equation $(L(v) \downarrow)^*$ derives from the tableau of linear equations under the current set of basic variables. With such a set of basic variables, we can simulate the linear equation $(L(v) \downarrow)^*$  in the configuration of $f'$. It is worth mentioning that, the inherit of the set of basic variables is also helpful for a SAT node $v$, since it is highly possible that the counterexample is still easily found for $f'$ with a similar pattern of assignment.

For an UNSAT leaf $v$, after we inherit the features mentioned above and form a simulating configuration for $f'$, it is not difficult to observe that, the precision of the numerical bound of the variables in $(L(v) \downarrow)^*$ is highly related to whether the proof of $\mathsf{UNSAT}$ is valid. The more precise the numerical bound is, the more likely $\mathsf{UNSAT}$ is to be inferred. Notice that the semantics of the configurations in $L(v)$ is equivalent to the conjunction of the initial configuration $\mathcal C_0$ and the assertions $\mathbf{Assert}(v)$, so we can compute the numerical bound straight via the behaviour of $f'$ and $\mathbf{Assert}(v)$ instead of the configurations in $L(v)$. Later we will discuss in detail how to achieve a balance of efficiency and precision in this step.

If the old verification result is $\mathsf{UNSAT}$, we just incrementally check whether the old $\mathsf{UNSAT}$ inferences still hold for $f'$, and for a leaf $v$ that cannot be immediately proved, we start a new Reluplex solving with the configuration $L(v)\downarrow$.  If the old verification result is $\mathsf{SAT}$, we first locate ourselves in the SAT leaf and see whether there is a counterexample in this branch. If not, we go through the leaves labelled with $\varepsilon$ from the nearest to the farthest from the SAT leaf. If still there is no counterexample, we should go back to the old UNSAT leaves and conduct incremental solving as the $\mathsf{UNSAT}$ setting. 

The framework of incremental SMT solving for DNN verification is shown in Fig.~\ref{fig:framework}. We show the main idea of our incremental SMT solving with the following example.

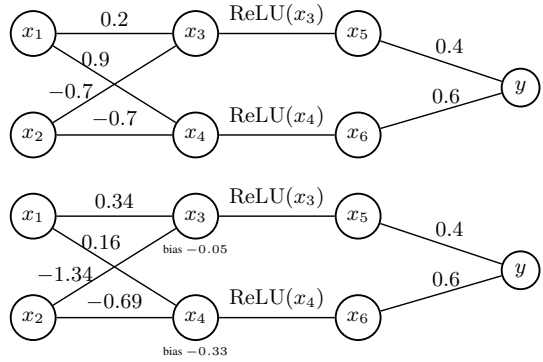
\begin{figure}[t]
  \centering
  \scalebox{0.9}{
 \begin{tikzpicture}[->,>=stealth,auto,node distance=1.2cm,semithick,scale=1,every node/.style={scale=1}]
	\tikzstyle{blackdot}=[circle,fill=black,minimum size=6pt,inner sep=0pt]
	\tikzstyle{state}=[minimum size=0pt,circle,draw,thick]
	\tikzstyle{stateNframe}=[minimum size=0pt]	
	\node[state](x1){$x_1$};
	\node[state](x2)[below of=x1,yshift=-0.3cm]{$x_2$};
	\node[state](x3)[right of=x1,xshift=1.2cm]{$x_3$};
	\node[state](x4)[right of=x2,xshift=1.2cm]{$x_4$};
	\node[state](y1)[right of=x3,xshift=1.2cm]{$x_5$};
	\node[state](y2)[right of=x4,xshift=1.2cm]{$x_6$};
	\node[state](y)[right of=y1,xshift=1.2cm,yshift=-0.8cm]{$y$};
	
		\node[state](1x1)[below of=x2]{$x_1$};
	\node[state](1x2)[below of=1x1,yshift=-0.3cm]{$x_2$};
	\node[state](1x3)[right of=1x1,xshift=1.2cm]{$x_3$};
	\node[state](1x4)[right of=1x2,xshift=1.2cm]{$x_4$};
	\node[state](1y1)[right of=1x3,xshift=1.2cm]{$x_5$};
	\node[state](1y2)[right of=1x4,xshift=1.2cm]{$x_6$};
	\node[state](1y)[right of=1y1,xshift=1.2cm,yshift=-0.8cm]{$y$};
%	\node[blackdot](d1)[above of=r11,xshift=0.6cm]{};
\node[stateNframe](bias1)[below of=1x3,yshift=0.7cm]{\tiny{bias $-0.05$}};
\node[stateNframe](bias2)[below of=1x4,yshift=0.7cm]{\tiny{bias $-0.33$}};

	\path (x1) edge	[-]						node {$0.2$} (x3)
		(x1) edge[-]							node[xshift=-0.6cm,yshift=0.15cm] {$0.9$} (x4)
		(x2) edge	[-]						node[xshift=-0.2cm,yshift=-0.35cm] {$-0.7$} (x3)
		(x2) edge[-]							node {$-0.7$} (x4)
			(x3) edge	[-]						node { $\mathrm{ReLU}(x_3)$} (y1)
		(x4) edge[-]							node {$\mathrm{ReLU}(x_4)$} (y2)
	(y1) edge[-]							node [xshift=-0.2cm]{$0.4$} (y)
			(y2) edge[-]							node[xshift=0.4cm] {$0.6$} (y)
			
		 (1x1) edge	[-]						node {$0.34$} (1x3)
		(1x1) edge[-]							node[xshift=-0.6cm,yshift=0.15cm] {$0.16$} (1x4)
		(1x2) edge	[-]						node[xshift=-0.2cm,yshift=-0.35cm] {$-1.34$} (1x3)
		(1x2) edge[-]							node {$-0.69$} (1x4)
			(1x3) edge	[-]						node { $\mathrm{ReLU}(x_3)$} (1y1)
		(1x4) edge[-]							node {$\mathrm{ReLU}(x_4)$} (1y2)
	(1y1) edge[-]							node [xshift=-0.2cm]{$0.4$} (1y)
			(1y2) edge[-]							node[xshift=0.4cm] {$0.6$} (1y)
			  ;
\end{tikzpicture}}
  \caption{The modified DNNs $f'$ (above) and $f''$ (below)} \label{fig:modified}
\end{figure}

\begin{example} \label{example:incremental}
In Fig.~\ref{fig:modified}, we show two modified DNNs $f'$ and $f''$ from the original DNN $f$ in Fig.~\ref{fig:example1}. The input region $X$ and the property $P$ are the same as those in Example~\ref{example:2}. 

We first incrementally verify the property $(f',X,P)$. Because there is a SAT node in the solving procedure of $f$, we first incrementally solve this SAT node $v_5$. We obtain the Assertions $\mathbf{Assert}(v_5)=x_4 \ge 0 \wedge x_3 \le 0$ and initialize the configuration of $f'$ restricted on this constraint. Fortunately, after one pivot, a counterexample $(0.714,0.204)^\mathrm{T}$ is found, and it outputs $\mathsf{SAT}$. Since $f'$ differs slightly from $f$, incremental SMT solving allows us to quickly find a counterexample of $f'$ which has exactly the same activation pattern as that of $f$, and we avoid a lot of local search on the other nodes.

The weight of the DNN $f''$ varies a lot from the original DNN $f$. In a similar way, we incrementally solve the SAT leaf $v_5$. However in this case, the solving result of $v_5$ becomes $\mathsf{UNSAT}$, so we have to solve the other nodes. We first go to $v_6$ whose label is empty, and see whether there is a counterexample. We heuristically first deal with $v_6$ instead of two UNSAT leaves $v_3$ and $v_5$ because we believe that leaves labelled with $\varepsilon$ are more likely to have a counterexample of the modified DNN than UNSAT leaves. In this example, the Reluplex solving result of $v_6$ is unfortunately $\mathsf{UNSAT}$, so we have to go back to the UNSAT leaves $v_3$ and $v_5$. Incremental solving for UNSAT leaves is different, because we can make use of the old $\mathsf{UNSAT}$ proofs. For the UNSAT leave $v_3$ here, we originally have a proof $\mathcal C_6 \to \mathsf{UNSAT}$. We aim to extract a configuration of $f''$ restricted on $\mathbf{Assert}(v_3)$ which simulates $\mathcal C_6$. Here we extract the set $\mathcal B_6$ of basic variables in $\mathcal C_6$, and transform the initial tableau for $f''$ to the form with $\mathcal B_6$ the basic variables. Also, the inference $\mathcal C_6 \to \mathsf{UNSAT}$ derives from a linear equation $\mathcal C_6^*$ in the tableau satisfying \eqref{eq:unsat1} or \eqref{eq:unsat1}, so we focus on this linear equation in the current tableau. Notice that tighter numerical bounds are beneficial for proving this $\mathsf{UNSAT}$, so we solve a linear programming with the optimization object the variables in this linear equation $\mathcal C_6^*$. With the tightened bounds, $\mathsf{UNSAT}$ is immediately inferred from $\mathcal C_6^*$. Similar procedures are also conducted on the UNSAT leaf $v_4$, and it has an immediate $\mathsf{UNSAT}$ proof, too. Now all the leaves in the tree $\mathcal{T}$ has the verification result $\mathsf{UNSAT}$, so our incremental SMT solving returns $\mathsf{UNSAT}$, i.e., the property $(f'',X,P)$ holds.
\end{example}

Example~\ref{example:incremental} shows that, when the modification is small, we can probably inherit the verification result of the corresponding leaves in the old solving procedure, and a lot of local search and unnecessary solving on $\mathsf{UNSAT}$ leaves can be avoided. However, when the modification is so large that it changes essentially the behaviour of the DNN, many verification results cannot be inherited and maybe even more local search is needed. Intuitively, whether our incremental SMT solving is more efficient is determined by how large the modification is.

\subsection{Main algorithm}

 \begin{algorithm}[t]
    \caption{Incremental SMT solving for DNN verification}
    \label{Alg:main}
    \begin{algorithmic}[1]
        \Require
        \Statex  The Reluplex solving procedure $\mathcal T=(V,E,r,L)$ of a safety property $(f,X,P)$ and the modified DNN $f'$
        \Ensure
        \Statex Return $\mathsf{SAT}$ if $(f',X,P)$ does not hold, or $\mathsf{UNSAT}$ otherwise.
        \Function{Verify}{$f',X,P,\mathcal T$}
            \State $U \gets \mathbf{DeepPoly}(f',X)$
            \If{$U \wedge \neg P = \bot$}  
            \State \Return $\mathsf{UNSAT}$
            \EndIf
            \State  $\mathcal C_0 \gets \mathbf{Initialize}(f',X,P,U)$ \Comment{standard Reluplex encoding according to the abstraction $U$}
       %     \State $V_\mathrm{u} \gets \{v \mid v \text{ was marked as uncertain in Line 2}\}$
              \For{$e \in E$}
                \If{$L(e) \cap U=\varnothing$} 
                \State $\mathcal T \gets \mathcal T \setminus \{e\}$
                \EndIf
              \EndFor
            \If{there is a SAT leaf $v_{\mathrm{SAT}} \in V$}
             \For{leaves $v$ labelled with $\mathsf{SAT}$ or $\varepsilon$}\Comment{with smaller distance from $v_\mathrm{SAT}$  takes priority}
           \State $\mathcal C \gets L(v) \downarrow$ \Comment{the current configuration $\mathcal C=(\mathcal B,T,R,l,u,\alpha)$}
           \State $l,u \gets \mathbf{DeepPoly}(f',X \wedge \mathbf{Assert}(v))$
           %\Comment{DeepPoly abstraction restricted on $\mathbf{Assert}(v_{\mathrm{SAT}})$}
           \State $T \gets \mathbf{GaussElimination}(\mathcal{C}_0,\mathcal{B})$
            \If{$\mathbf{Reluplex}(\mathcal C)=\mathsf{SAT}$} 
            \State \Return $\mathsf{SAT}$ \Comment{Reluplex solving from the configuration $\mathcal C$}
            \EndIf
            \EndFor
            \EndIf
            \For{UNSAT leaves $v \in V$}   
            \If{$\mathbf{Solve}(v)=\mathsf{SAT}$ }
            \State \Return $\mathsf{SAT}$
            \EndIf
            \EndFor
     \State \Return $\mathsf{UNSAT}$ 
       \EndFunction
\end{algorithmic}\end{algorithm}

\begin{algorithm}[t]
    \caption{The function $\mathbf{Solve}(v)$ for an UNSAT node $v$}
    \label{Alg:solve}
    \begin{algorithmic}[1]
        \Require
        \Statex  The Reluplex solving procedure $\mathcal T=(V,E,r,L)$, the modified DNN $f'$, and an UNSAT leaf node $v \in V$
        \Ensure
        \Statex Return $\mathsf{SAT}$ if $(f',X \wedge \mathbf{Assert}(v),P)$ does not hold, or $\mathsf{UNSAT}$ otherwise.
        \Function{Solve}{$f',\mathcal T,v$}
           \State $\mathcal C \gets L(v) \downarrow$ \Comment{$\mathcal C=(\mathcal B,T,R,l,u,\alpha)$}
           \State $l,u \gets \mathbf{DeepPoly}(f',X \wedge \mathbf{Assert}(v))$
           \Comment{DeepPoly abstraction restricted on $\mathbf{Assert}(v)$}
           \If{$\mathbf{LP}(f'\wedge X \wedge \neg P \wedge \mathbf{Assert}(v))$ is infeasible}
          \State \Return $\mathsf{UNSAT}$
           \EndIf
           \For{$x_i \in \mathrm{Var}((L(v) \downarrow)^*)$} 
          \State  $(l_i,u_i) \gets \mathbf{LP}(f'\wedge X \wedge \neg P \wedge \mathbf{Assert}(v),x_i)$
            \EndFor
            \State $T \gets \mathbf{GaussElimination}(\mathcal{C}_0,\mathcal{B})$%\Comment{Transfer the basic variables the same as $L(v) \downarrow$}
            \If{$((L(v) \downarrow)^*(T),l,u)$ infers $\mathsf{UNSAT}$}
            \State \Return $\mathsf{UNSAT}$ \Comment{check whether the old $\mathsf{UNSAT}$ proof still works}
       %     \State $V_\mathrm{u} \gets \{v \mid v \text{ was marked as uncertain in Line 2}\}$
     \Else ~\Return $\mathbf{Reluplex}(\mathcal C)$ 
     \EndIf
       \EndFunction
\end{algorithmic}\end{algorithm}

 We show the main algorithm in Alg.~1. We are given the Reluplex solving procedure $\mathcal T=(V,E,r,L)$ of a safety property $(f,X,P)$ and the modified DNN $f'$.
 \paragraph{Initialization.}
 We first initialize the configuration for $f'$.
 %and cut the unsatisfiable branches in $\mathcal T$. 
 Since the weights in the modified DNN $f'$ are different from those in $f$, we have to recalculate the numerical bound for the initial configuration of $f'$. Usually a sound but incomplete method is invoked to compute the initial bound, and in this work we choose the abstract domain DeepPoly~\cite{deeppoly}. We denote by $\mathbf{DeepPoly}(f',X) \subseteq \mathbb R^{\mathcal X}$ the concretization of the DeepPoly abstraction for the DNN $f'$ on the input set $X$. Since DeepPoly is sound, the algorithm outputs $\mathsf{UNSAT}$ if DeepPoly has successfully verified the property.

 \paragraph{Pruning of invalid edges.}
 After the initialization, we may infer that some edges in the current tree $\mathcal T$ can be cut.
 An edge $e \in E$ is eliminated from the tree $\mathcal T$, denoted by $\mathcal T \setminus \{e\}$, if its label $L(e)$ contradicts with the result of the DeepPoly calculation, i.e. $L(e) \cap \mathbf{DeepPoly}(f',X)=\varnothing$ and the sub-tree beneath the edge $e$ is deleted from $\mathcal T$ at the same time. 
% \lz{again, refer to examples or add examples to illustrate the pruning}

\paragraph{SAT nodes}
If the verification result of $f$ is $\mathsf{SAT}$, there is a unique SAT leaf and possibly some leaves labelled with $\varepsilon$, since $\mathsf{SAT}$ is always inferred from a counterexample witnessed by an assignment, and the verification terminates once we find such an assignment, probably with some branches not solved. 
% For the SAT leaf or the leaves labelled with $\varepsilon$,
For these leaves, we invoke a simple Reluplex solving. In this situation, we always first solve the SAT leaf, then the leaves labelled with $\varepsilon$, and the UNSAT leaves will be left for last. This order is following the intuition that the verification results of the leave does not change for the modified DNN $f'$ in most cases, and the old UNSAT leaves are less likely to have a counterexample of $f'$ than the unsolved leaves. 
% For the leaves labelled with $\varepsilon$, we propose a heuristic way to order them for solving. The idea is that the leaves that are near the SAT leaf is more likely to have a counterexample than those far away in the tree $\mathcal T$, even if the solving result of the original SAT leaf is $\mathsf{UNSAT}$. For a labelled binary tree $\mathcal T=(V,E,r,L)$ and two nodes $v,v' \in V$, we define their distance as
For the leaves labelled with $\varepsilon$, we solve them in a heuristic order following the observation that the leaves near the SAT leaf is more likely to have a counterexample than those far away in the tree $\mathcal T$, even if the solving result of the original SAT leaf is $\mathsf{UNSAT}$. Specifically, we measure the distance of two nodes $v$ and $v'$ in a labelled binary tree $\mathcal{T}$ by
\begin{align*}
    d_\mathcal{T} (v,v')= |(\mathbf{Assert}(v) \setminus \mathbf{Assert}(v')) \cup (\mathbf{Assert}(v') \setminus \mathbf{Assert}(v))|,
\end{align*}
the cardinality of the symmetric difference between the sets $\mathbf{Assert}(v)$ and $\mathbf{Assert}(v')$. Before we start to solve the leaves labelled with $\varepsilon$, we calculate the distance between the SAT leaf and each of them, and sort them according to this distance from the smallest to the largest as the solving order.

\paragraph{UNSAT nodes} 
 For an UNSAT node $v$, we incrementally solve it by checking whether the old proof of $\mathsf{UNSAT}$ still works for the modified DNN $f'$. We state the detailed procedure $\mathbf{Solve}(v)$ in Alg.~\ref{Alg:solve}.
 Here we locate ourselves in $L(v) \downarrow$, the configuration that infers $\mathsf{UNSAT}$. 
 Based on the basic variables $\mathcal{B}$ in this configuration,
 we transfer the tableau of the initial configuration $\mathcal C_0$ for $f'$ into the form with the same basic variables by Gauss Elimination.
 Recall that an $\mathsf{UNSAT}$ Reluplex inference derives from Inequality $\eqref{eq:unsat1}$ or $\eqref{eq:unsat2}$, we further focus on the linear equation which immediately infers $\mathsf{UNSAT}$, denoted by $(L(v) \downarrow)^*$. Here a key observation is that, if the numerical bounds $l$ and $u$ are sufficient for inferring Inequality $\eqref{eq:unsat1}$ or $\eqref{eq:unsat2}$, then tighter bounds $l'$ and $u'$ with $l' \ge l$ and $u' \le u$ also work for proving $\mathsf{UNSAT}$. 
 The tighter that the numerical bounds of the involved variables are, the more likely $\mathsf{UNSAT}$ is to be inferred. If we focus on the linear equation $(L(v) \downarrow)^*$, the involved variables are naturally those that appear in it, and we denote the set of these variables as $\mathrm{Var}((L(v) \downarrow)^*)$. The numerical bounds obtained by DeepPoly are not tight enough, because the design of DeepPoly does not allow a backward analysis, and an assertion $x_j \ge 0$ (or $x_j \le 0$) in $\mathbf{Assert}(v)$ is not able to refine the bounds of the neurons in the same layer or the previous layers. To obtain comparatively tight numerical bounds, we may use linear programming. Except for uncertain ReLU relations, the other constraints are all linear equations or inequalities, and we follow a standard way \cite{planet} to over-approximate an uncertain ReLU relation $x_j=\mathrm{ReLU}(x_i),x_i \in [l_i,u_i]$ with the linear inequalities
 \begin{align} \label{eq:linearapproximation}
     x_j \ge 0,~ x_j \ge x_i,~x_j \le \frac{u_i(x_i-l_i)}{u_i-l_i}.
 \end{align}
 With the constraints of the input region $X$, the negation $\neg P$ of the property, the conjunction of assertions $\mathbf{Assert}(v)$, and the behaviour of the modified 
 DNN $f'$ with linear approximation, we construct the LP problem, denoted by $\mathbf{LP}(f'\wedge X \wedge \neg P \wedge \mathbf{Assert}(v))$. If this LP is feasible, we optimize $\min x_i$ and $\min -x_i$ for $x_i \in \mathrm{Var}((L(v) \downarrow)^*)$ to obtain tighter bounds for these involved variables, and check whether the equation that simulates $(L(v) \downarrow)^*$ in the current tableau $T$ along with the current numerical bounds infers $\mathsf{UNSAT}$. If not, we continue to conduct a Reluplex solving for $v$ from the configuration $L(v) \downarrow$ in which the tableau and the numerical bounds are replaced with the current ones.

Alg.~1 is sound and complete. We conclude the correctness of Alg.~\ref{Alg:main} in the following theorem:

\begin{theorem}
Alg.~\ref{Alg:main} is sound and complete, i.e., it returns $\mathsf{UNSAT}$ iff the property $(f',X,P)$ holds. 
\end{theorem}
\begin{proof}
First we claim that for an UNSAT leaf $v$, $\mathbf{Solve}(v)$ return $\mathsf{UNSAT}$ iff the property $(f',X\wedge Assert(v),P)$ holds. 
In Line~8 of Alg.~\ref{Alg:main}, the semantics of the current configuration $\mathcal C$ is exactly $f'\wedge X \wedge \mathbf{Assert}(v)\wedge \neg P$, since $X$, $\neg P$, and the assertions $\mathbf{Assert}(v)$ have been encoded in $L(v)$, and the tableau transformation in Line~8 encodes the affine transformations in $f'$, with the ReLU relations $R$ unchanged. 
The numerical bounds $l$ and $u$ is an over-approximation of $f'\wedge X \wedge \mathbf{Assert}(v)\wedge \neg P$, so they do not violate this semantics. 
If $\mathbf{Solve}(v)$ returns $\mathsf{UNSAT}$ in Line~10 of Alg.~\ref{Alg:solve}, then an $\mathsf{UNSAT}$ proof is obtained for $\mathcal C$, which implies that the property $(f',X\wedge \mathbf{Assert}(v),P)$ holds. 
If $\mathbf{Solve}(v)$ returns $\mathsf{UNSAT}$ in Line~11 of Alg.~\ref{Alg:solve}, since Reluplex is sound and complete, the property $(f',X\wedge\mathbf{Assert}(v),P)$ also holds. Now we assume that $(f',X\wedge\mathbf{Assert}(v),P)$ holds, then either a quick proof is given in Line~10 of Alg.~\ref{Alg:solve}, or Reluplex solves it in Line~11, so it must output $\mathsf{UNSAT}$.

For the leaves $v$ that are not UNSAT, similarly in Line~10 of Alg.~\ref{Alg:main}, the semantics of the current configuration $\mathcal C$ is $f'\wedge X \wedge \mathbf{Assert}(v)\wedge \neg P$. $\mathbf{Reluplex}(\mathcal C)$ returns $\mathsf{UNSAT}$ iff $(f',X\wedge\mathbf{Assert}(v),P)$ holds. Therefore, we have  Alg.~\ref{Alg:main} returns $\mathsf{UNSAT}$ iff for all  leaves $v$, $(f',X\wedge\mathbf{Assert}(v),P)$ holds. From the construction of the solving procedure, naturally we have $\bigvee_{v \text{ is a leaf}} \mathbf{Assert}(v)=\top$, so for all  leaves $v$, $(f',X\wedge\mathbf{Assert}(v),P)$ holds iff $(f',X,P)$ holds. We complete the proof.
\end{proof}

\subsection{Optimization Techniques}
In the implementation, we propose the following optimization techniques to improve the efficiency.

\subsubsection{3.4.1 Using an attack method for the SAT leaf.}
 If there is a SAT leaf node $v$ in $\mathcal T$, we first heuristically use a gradient based attack method (e.g. PGD~\cite{pgd}) to search a counterexample. Here we do not invoke the attack method directly to the DNN $f'$, but $f'$ under the constraints of the assertions from the root leading to this SAT leaf $v$. When calculating the gradient of $f'$, we always use the assertions in $\mathbf{Assert}(v)$ to determine the derivative of the involved ReLU functions. In this way, we are searching a counterexample which has the same activation pattern in $\mathbf{Assert}(v)$ as with the node $v$. If no counterexample is found, we still have to run the Reluplex solving $\mathbf{Reluplex}(v)$. 
 
 \subsubsection{3.4.2 Efficient linear programming for tightening numerical bounds.} In Line~6 \& 7 of Alg.~\ref{Alg:solve}, linear programming executes for every variable that appears in the linear equation $(L(v) \downarrow)^*$ to tighten its numerical bounds.
 When the number of variables in $(L(v) \downarrow)^*$ is very large, the LP solving is quite time-consuming. 
 Here we can change the optimization object in the linear programming to the input neurons, i.e., for all the input neurons $x_i$, we optimize $\min x_i$ and $\min -x_i$ to obtain tighter bounds for the input region. 
 With these tightened input bounds, we run DeepPoly again to obtain the numerical bounds of the variables in $(L(v) \downarrow)^*$.
 Obviously the bounds obtained by this efficient linear programming is less precise, since the abstraction of the input region is a box abstraction under the constraints $f'\wedge X \wedge \neg P \wedge \mathbf{Assert}(v)$, and the abstraction of DeepPoly starting with this box abstraction is naturally less precise than straight LP on the target variables in $(L(v) \downarrow)^*$.
 However, for the DNNs whose input dimension is not large, it significantly reduces the number of LP problems to be solved.
 Another advantage of this technique is that the DeepPoly abstraction gives all the variables tightened bounds, which benefits the Reluplex solving $\mathrm{Reluplex}(v)$ when $\mathsf{UNSAT}$ cannot be immediately inferred in this incremental way.
 
 \subsubsection{3.4.3 Linear approximation encoding for uncertain ReLU relations in the tableau.} This technique is generally for the underlying solver, namely Marabou here. When the behaviours of DNNs are encoded into constraints, uncertain ReLU relations are only encoded in the ReLU relation $R$ in the configuration, and such relations are only used in the local search in which the assignments are adjusted to satisfy these relations before they are chosen to be split or refined to become definitely activated/deactivated. Intuitively, such ReLU relations are not sufficiently exploited in the local search. Here we additionally encode the linear approximation~\eqref{eq:linearapproximation} of the uncertain ReLU relations in the tableau $T$.
 This significantly refines the semantics of the tableau $T$ of the configuration, and thus makes the search space restricted and reduces invalid local search. It is worth mentioning that, the constraints of such linear approximation are managed dynamically, i.e., when the numerical bounds of an uncertain neuron is refined by an assertion, we renew the linear approximation~\eqref{eq:linearapproximation} with the refined bounds. Since pivot operations may be conducted on such constraints and they will interact with other linear constraints, we also need to simultaneously renew those constraints when such refinement occurs.
 
 \subsubsection{3.4.4 Lazy incremental solving for UNSAT nodes.} For an UNSAT leaf $v$, our incremental solving in Alg.~\ref{Alg:solve} needs to transfer the tableau to the form with exactly the same basic variables as $L(v) \downarrow$. In the testing of DeepInc, we find that such transformation does not really benefit the incremental solving in most cases, i.e., $\mathsf{UNSAT}$ can still be immediately inferred even if the set of basic variables are not the same as that of $L(v) \downarrow$. In the implementation, we straight invoke efficient linear programming for tightening numerical bounds (see Optimization~3.4.2) and solve from the current tableau and these tightened bounds. In most cases, the running time of lazy incremental solving for UNSAT nodes is reduced.

\section{Experimental Evaluation} \label{sec:experiment}
In this section we present the experimental evaluation of our incremental SMT solver DeepInc for DNN verification.
All the experiments are run on an Ubuntu 20.04 laptop with AMD R9 5900HS @ 3.00GHz (8 cores) and 16G RAM, and  16 sub-processes concurrently are used at most. For a fair comparison, only Optimization 3.4.2 and 3.4.4 are used in the experiment.

\begin{figure}[t]
    \centering
    {%
        \includegraphics[width=0.45\linewidth]{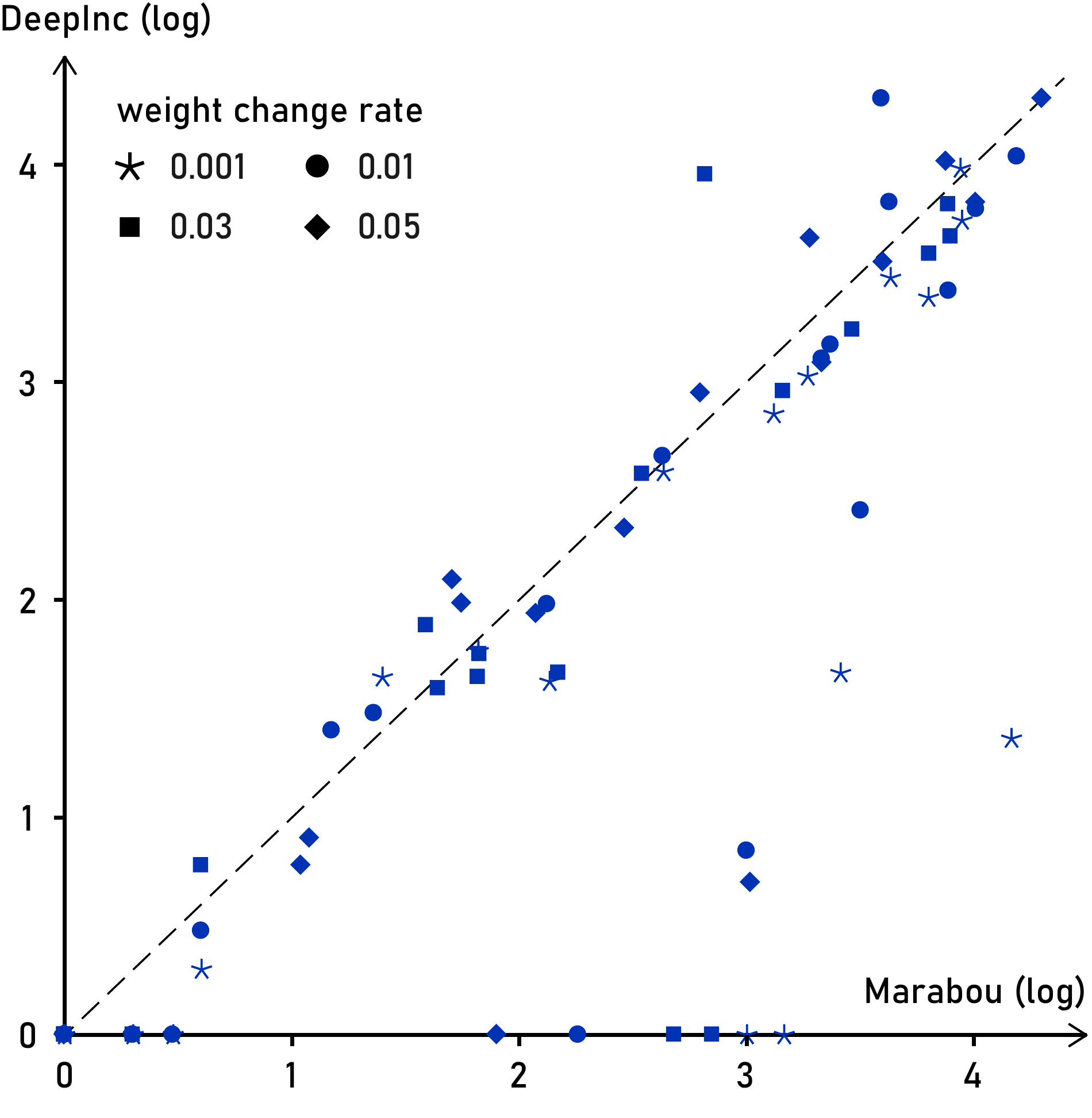}}
    \quad
    {%
        \includegraphics[width=0.45\linewidth]{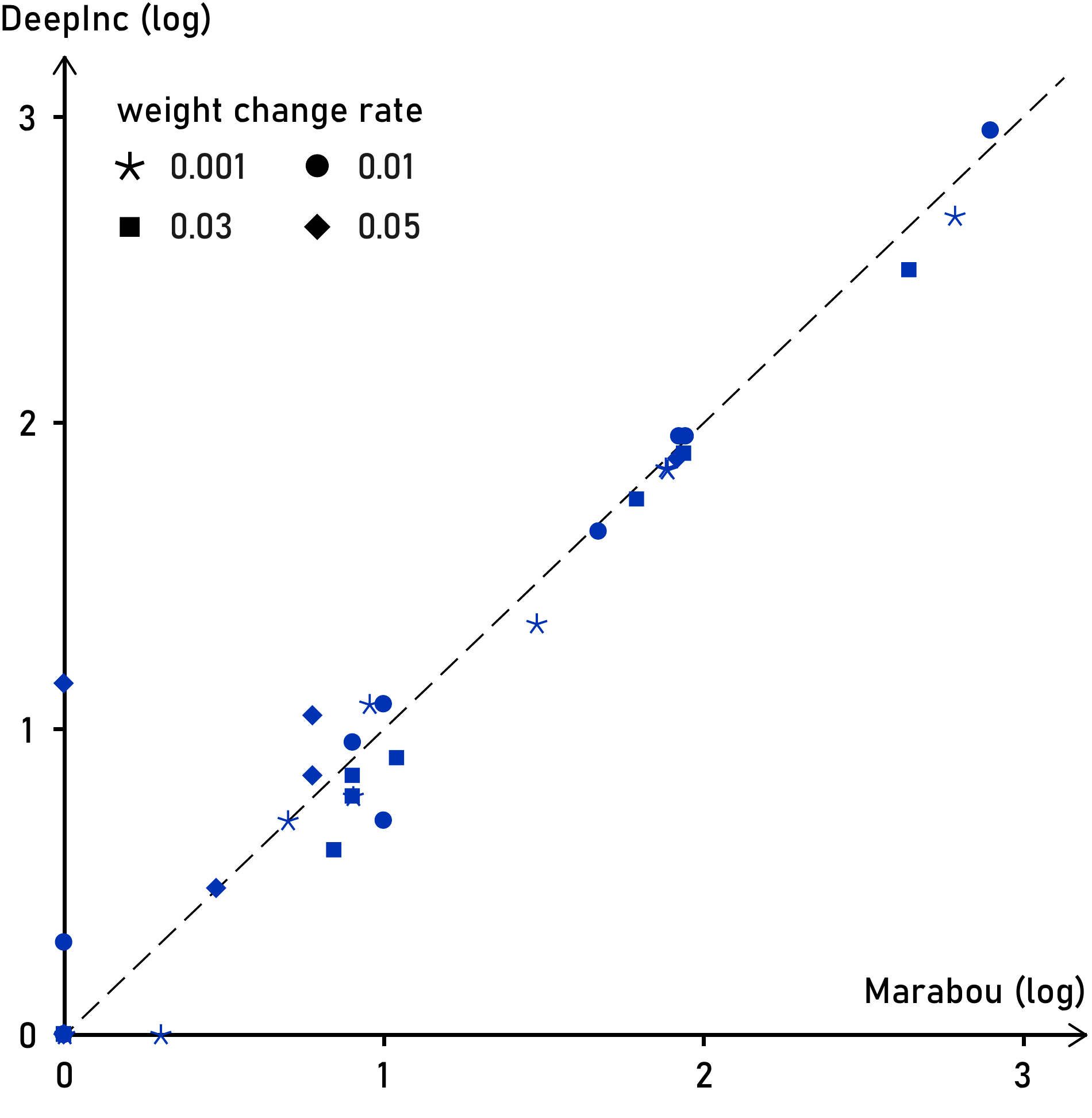}}
    \caption{The running time (in log sec) of DeepInc and Marabou on DNNs with all the weights changed on four safety properties (left) and three local robustness properties (right)} \label{fig:ex1}
\end{figure}

\textbf{Networks}:
We randomly choose five DNNs from 45 ACAS Xu networks for evaluation. ACAS Xu is short for Airborne Collision Avoidance System for Unmanned Aircraft~\cite{reluplex}, which is aimed to avoid airborne collisions for unmanned aircrafts by using an observation table to make decisions for the aircraft. 
%The observation table is compressed by 45 neural networks  in the  ACAS Xu system~\cite{comprass-lookup-table}; 
Each of these networks consists of $6$ hidden layers with $50$ neurons in each hidden layer, and ReLU activation functions are applied to all hidden layer neurons. The networks take a feature vector of size $5$ as input that encodes the speed and relative position of both the intruder and ownership and output the prediction scores of five advisories. 

\textbf{Properties}:
We choose four safety properties, denoted as $\varphi_1$, $\varphi_2$, $\varphi_3$, and $\varphi_4$ from \cite{reluplex} that the networks of ACAS Xu are supposed to follow, and three local robustness properties whose radius are $0.05$, $0.075$, and $0.01$ to evaluate our solver. 

\textbf{Modifications}: 
We consider two kinds of randomly generated modification on the weight of the DNNs. The first is all the weights changed within a weight change rate $\gamma$. Under the weight change rate $\gamma$, the value $w$ can be modified to a number in the interval $I(w,\gamma):=[(1-\gamma)w,(1+\gamma)w]$. Given the weight change rate $\gamma$, the weight in the modified DNN is obtained by a random sampling according to the uniform distribution on $I(w,\gamma)$. In our experiment, the weight change rate $\gamma$ is set to be $0.001$, $0.01$, $0.03$ and $0.05$. The second is partial weight change, i.e., only part of the weights in the original DNN is modified. The percentage of changed weights is set as $10\%$, $30\%$ and $50\%$ in our experiment, and the weights to be changed are randomly chosen. In partial change, we also have the restriction of weight change rate on the changed weights, which are $0.01$, $0.03$ and $0.05$ in the experiment.

In the following, we are going to answer the research questions below.
\begin{itemize}
    \item[{RQ1:}] Can DeepInc verify safety properties of a slightly modified DNN more efficiently compared
    with verification without an incremental design?
    \item[{RQ2:}] To what extent of the changes for the modified DNN can DeepInc retain its advantage in efficiency?
    \item[{RQ3:}] For the instances where DeepInc is less efficient, is there any essential change in the behaviour of the DNN that prevents a key inference from being inherited in the incremental SMT solving?
   % \item[{\bf RQ4:}] Is there a underlying relation between DNN local robustness verification and DNN testing (especially the test selection)? % between the maximum robustness radius and the testing prioritising metric?
\end{itemize}

We run both Marabou and DeepInc on the randomly generated modified DNNs, where run of DeepInc also has the old procedure of the original DNNs given by Marabou as input. Timeout (TO) is set to be $20\,000$ seconds for each SMT solving instance. The experimental results are shown in Fig.~\ref{fig:ex1} and Fig.~\ref{fig:ex2}. It turns out that, the verification results of the original DNN and all its modified DNNs for the same property are the same for all the cases.

For overall efficiency, DeepInc outperforms Marabou in efficiency in $81\%$ of the cases, and the total running time of Marabou is $2.003$ times that of DeepInc. DeepInc and Marabou both has one timeout result on the same case. From Fig.~\ref{fig:ex1}, we can see that DeepInc shows more advantage in efficiency on the four safety properties, with $11$ cases more than $100$ times faster, but their performances on local robustness verification are comparatively close. The acceleration rate on $\mathsf{UNSAT}$ cases and $\mathsf{SAT}$ cases is $61\%$ and $4101\%$, respectively, showing that DeepInc is outstanding in incrementally searching for counterexamples.

\noindent\doublebox{
\begin{minipage}{0.94\linewidth}
 {Answer RQ1:} DeepInc is more efficient in up to $81\%$ of the cases with an acceleration of speed up to $100\%$, showing significant advantage in efficiency. 
\end{minipage}
}

We are interested in the performance of DeepInc in different modification size. An important index here is the percentage of UNSAT leaves on which $\mathsf{UNSAT}$ is proved in the incremental pattern i.e., Line 11 of Alg.~\ref{Alg:solve} is not run. For the four weight change rate $0.001$, $0.01$, $0.03$ and $0.05$ of all weights changed, this percentage is $89.01\%$, $85.77\%$, $89.43\%$ and $85.84\%$, respectively. For the experiments of partial weight change, this percentage is $88.45\%$. This shows that the performance of the core of our incremental solving does not significantly decrease when the weight change rate changes gradually from $0.001$ to $0.05$, and it is also stable for partial weight change.

\noindent\doublebox{
\begin{minipage}{0.94\linewidth}
 {Answer RQ2:} DeepInc maintains its efficiency advantages in both modification patterns of global weight change and partial weight change, with the weight change rate up to $5\%$. The core incremental solving module is effective and stable towards different weight modification. % in comparison with similar tools.
 %much closer to the exact maximum robustness radius compared to \provero. Also, \deeppac can depict local robustness more accurately and produce meaningful analysis results when estimating the robustness rate with a fixed radius.
\end{minipage}
}

We further investigate the cases in which DeepInc is less efficient. For such cases whose output is $\mathsf{SAT}$, we check whether the counterexample found in the modified DNN is in the branch of the original SAT leaf $v_{\mathrm{SAT}}$ by matching the ReLU activation pattern of the counterexample with the assertions in $\mathbf{Assert}(v_{\mathrm{SAT}})$. We find that all the counterexamples in such cases are not in the branch of $v_{\mathrm{SAT}}$, which implies that the output of $\mathbf{Reluplex}(v_{\mathrm{SAT}})$ is all $\mathsf{UNSAT}$, and the counterexample is found in another branch, probably after a great amount of local search in many branches which outputs $\mathsf{UNSAT}$. For the cases in which DeepInc outputs $\mathsf{SAT}$ faster than Marabou, over $90\%$ of the counterexamples are in the same branch of the SAT leaves in the old verification procedure. For the $\mathsf{UNSAT}$ cases, it is not easy to observe such essential change, so we obtain the percentage of UNSAT leaves on which UNSAT is proved in the incremental pattern for the cases that DeepInc is less efficient, which is $86.20\%$, $85.09\%$, $65.64\%$, $84.43\%$, and $83.41\%$, respectively. Compare the percentage here with above, we find that four of them decrease a little bit, and that one for the change rate being $0.03$ decreases significantly. This shows that in these cases fewer $\mathsf{UNSAT}$ proofs can be proved in the incremental pattern.

\noindent\doublebox{
\begin{minipage}{0.94\linewidth}
 {Answer RQ3:} In the cases where DeepInc is less efficient, modification of DNNs significantly change the behaviours of DNNs so that proofs in the incremental pattern mostly fail, especially for such cases that output $\mathsf{SAT}$. % in comparison with similar tools.
 %much closer to the exact maximum robustness radius compared to \provero. Also, \deeppac can depict local robustness more accurately and produce meaningful analysis results when estimating the robustness rate with a fixed radius.
\end{minipage}
}

\begin{figure}[t]
    \centering
        \includegraphics[width=0.45\linewidth]{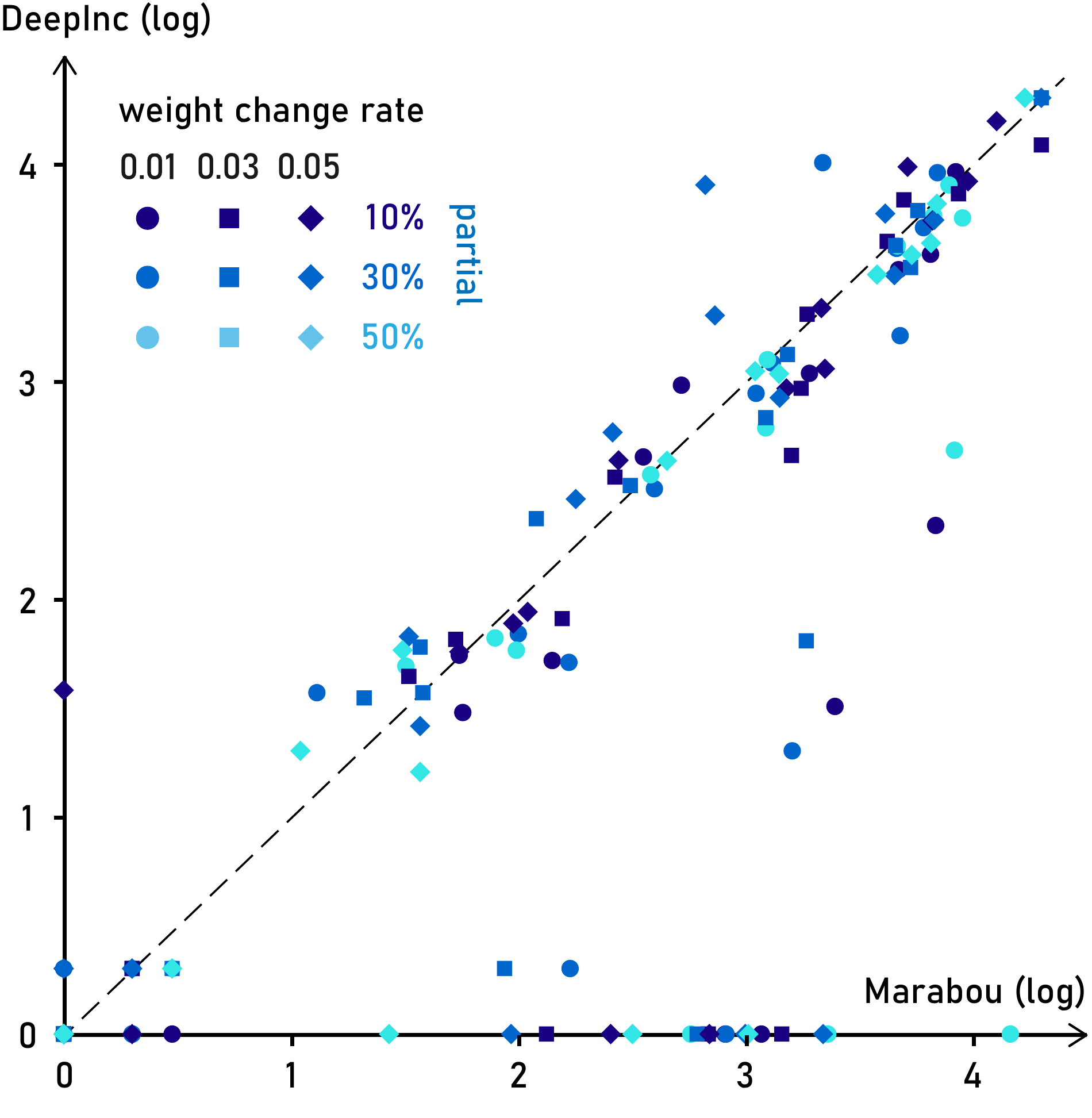}
    \caption{The running time (in log sec) of DeepInc and Marabou on DNNs with partial weight change
on four safety properties} \label{fig:ex2}
\end{figure}

\section{Multi-Objective DNN Repair} \label{sec:repair}
In this section, we propose the \emph{multi-objective DNN repair problem}, and present an algorithm based on our incremental techniques.

% Please add the following required packages to your document preamble:
% \usepackage{multirow}
\begin{table}[t]
\setlength\tabcolsep{10pt}
\begin{tabular}{|cc|cccccc|}
\hline
\multicolumn{2}{|c|}{Method}                                           & DNN  & $\varphi_2$ & $\varphi_1$ & $\varphi_3$ & $\varphi_4$ & robustness                       \\ \hline
\multicolumn{2}{|c|}{{CARE}}                          & 2\_1 &$\times$ &\checkmark &\checkmark &\checkmark & 90\%
        \\
\multicolumn{2}{|l|}{}                                               & 2\_3 &\checkmark &\checkmark &$\times$ &$\times$ & 70\%                      \\
\multicolumn{2}{|l|}{}                                               & 2\_4 &$\times$ &\checkmark &\checkmark &$\times$ & 70\%                      \\
\multicolumn{2}{|l|}{}                                               & 2\_6 &$\times$ &$\times$ &TO &TO & 60\%                      \\
\multicolumn{2}{|l|}{}                                               & 2\_8 &$\times$ &$\times$ &$\times$ &TO & 70\%                      \\ \hline
\multicolumn{2}{|c|}{{PRDNN}}                         & 2\_1 &\checkmark &\checkmark &$\times$ &\checkmark & 0\%                       \\
\multicolumn{2}{|l|}{}                                               & 2\_3 &$\times$ &\checkmark &\checkmark &$\times$ & 0\%                       \\
\multicolumn{2}{|l|}{}                                               & 2\_4 &\checkmark &\checkmark &$\times$ &$\times$ & 0\%                       \\
\multicolumn{2}{|l|}{}                                               & 2\_6 &$\times$ &TO &$\times$ &$\times$ & 0\%                       \\
\multicolumn{2}{|l|}{}                                               & 2\_8 &TO &TO &$\times$ &$\times$ & 0\%                       \\ \hline
\multicolumn{2}{|c|}{{ART}}                           & 2\_1 &\checkmark &\checkmark &\checkmark &\checkmark & 70\%                      \\
\multicolumn{2}{|l|}{}                                               & 2\_3 &\checkmark &\checkmark &\checkmark &\checkmark & 70\%                      \\
\multicolumn{2}{|l|}{}                                               & 2\_4 &\checkmark &\checkmark &\checkmark &\checkmark & 60\%                      \\
\multicolumn{2}{|l|}{}                                               & 2\_6 &\checkmark &\checkmark &\checkmark &\checkmark & 70\%                      \\
\multicolumn{2}{|l|}{}                                               & 2\_8 &\checkmark &\checkmark &\checkmark &\checkmark & 70\%                      \\ \hline
\multicolumn{1}{|l|}{{ours}} & {0.01} & 2\_1 &$\times$ &\checkmark &\checkmark &\checkmark & 70\%                      \\
\multicolumn{1}{|l|}{}                       &                       & 2\_3 &$\times$ &\checkmark &\checkmark &\checkmark & 70\%                      \\
\multicolumn{1}{|l|}{}                       &                       & 2\_4 &$\times$ &\checkmark &\checkmark &\checkmark & 70\%                      \\
\multicolumn{1}{|l|}{}                       &                       & 2\_6 &$\times$ &TO &\checkmark &\checkmark & 70\%                      \\
\multicolumn{1}{|l|}{}                       &                       & 2\_8 &$\times$ &$\times$ &\checkmark &\checkmark & 70\%                      \\ \cline{2-8} 
\multicolumn{1}{|l|}{}                       & {0.03} & 2\_1 &\checkmark &\checkmark &\checkmark &\checkmark & 80\%                      \\
\multicolumn{1}{|l|}{}                       &                       & 2\_3 &$\times$ &\checkmark &\checkmark &\checkmark & 70\%                      \\
\multicolumn{1}{|l|}{}                       &                       & 2\_4 &\checkmark &\checkmark &\checkmark &\checkmark & 70\%                      \\
\multicolumn{1}{|l|}{}                       &                       & 2\_6 &TO &TO &\checkmark &\checkmark & 70\%                      \\
\multicolumn{1}{|l|}{}                       &                       & 2\_8 &TO &TO &\checkmark &\checkmark & 80\%                      \\ \cline{2-8} 
\multicolumn{1}{|l|}{}                       & {0.05} & 2\_1 &\checkmark &\checkmark &\checkmark &\checkmark & 80\%                      \\
\multicolumn{1}{|l|}{}                       &                       & 2\_3 &\checkmark &\checkmark &\checkmark &\checkmark & 70\%                      \\
\multicolumn{1}{|l|}{}                       &                       & 2\_4 &\checkmark &\checkmark &\checkmark &\checkmark & 70\%                      \\
\multicolumn{1}{|l|}{}                       &                       & 2\_6 &$\times$ &\checkmark &\checkmark &\checkmark & 80\%                      \\
\multicolumn{1}{|l|}{}                       &                       & 2\_8 &$\times$ &$\times$ &\checkmark &\checkmark & 80\%                      \\ \hline
\end{tabular}
\caption{Comparison of our multi-objective repair and state-of-the-art repair methods, where $\times$ means that the property is violated, \checkmark means that the property holds, TO represents timeout of $20\,000$ seconds. For local robustness properties we record the percentage of the robustness properties that still hold.} \label{tab:repair}
\end{table}

DNN repair is aimed to modify a given DNN to fix certain aspects of its behaviour, including violation of safety properties, lack of fairness, and backdoor attack. In many repair methods, only the weight of the original DNN is modified, which exactly matches the incremental situation considered in this work. In \cite{care}, DNN repair is formally defined as follows:

\begin{definition}
Given a DNN $f$, a input set $X$, and a property $\varphi$, the DNN repair problem is to construct a neural network $f'$ such that $f'$ satisfies the property $\varphi$ on $X$, and  the semantic distance between $f$ and $f'$ is minimized. For classification neural networks, the semantic distance is defined as the probability that $f$ and $f'$ gives different classification results, i.e.,
$
P(\{x \in X \mid C_f(x) \ne C_{f'}(x)\})
$.
\end{definition}

In practice, DNNs are often subject to multiple safety properties. For instance, in detection scenarios, multiple objects detected should satisfy the robustness properties. In such situation, we are interested in the multi-objective DNN repair problem, detailed in the definition below:

\begin{definition}
Given a DNN $f$, a violated safety property $\varphi$, and a set of safety properties $\Phi$, the multi-objective DNN repair problem is to construct a neural network $f'$ such that $f'$ satisfies the safety property $\varphi$, with the safety properties in $\Phi$ not violated.
\end{definition}

If the set $\Phi$ contains only a small number of safety properties, this problem can be solved by state-of-the-art DNN repair methods like CARE~\cite{care}, ART~\cite{art}, and PRDNN~\cite{prdnn}. However, when the set $\Phi$ is very large (or even uncountable), it is difficult to consider encoding all the safety properties in $\Phi$ as part of optimization object in the repair. Here the experimental results of DeepInc in Section~\ref{sec:experiment} motivate us to add constraints of weight change rate in the repair, since even weight change rate is up to $0.05$, all the verification results of the modified and the original DNN are exactly the same. If the repair is only aimed to fix the violation of the property $\varphi$ under a restriction of weight change rate, then the properties in $\Phi$ are highly likely to be maintained.

We design an algorithm for multi-objective DNN repair based on the repair method ART~\cite{art}. In the gradient descent of ART, we add a clamp to restrict the weight change within a given rate. Because ART measures whether the property $\varphi$ holds with an abstract domain $\mathcal D$, we change the terminating condition from $\epsilon_\mathcal D=0$ to $\epsilon_\mathcal D \le \theta$ where $\theta$ is a positive hyper-parameter and an incremental SMT solving is immediately conducted after this condition is satisfied. We output the current modified DNN if the incremental SMT solving returns $\mathsf{UNSAT}$. If we keep the terminating condition $\epsilon_\mathcal D=0$ with the restriction of weight change rate, many running examples are experimentally not terminating.

On five ACAS Xu DNNs where the safety property $\varphi_2$ is violated, we compare our algorithm of multi-objective DNN repair with state-of-the-art repair methods CARE, PRDNN and ART. The set $\Phi$ consists of three safety properties $\varphi_1$, $\varphi_3$ and $\varphi_4$, and ten randomly chosen  local robustness properties, with the robustness radius $0.025$ for five properties and $0.05$ for the other five. For our method, we try three different restrictions of weight change rate $0.01$, $0.03$ and $0.05$. The experimental results are shown in Table~\ref{tab:repair}. We can see from the results that, when the weight change rate is restricted to a very small value like $0.1$, then it is difficult to fix the violated property $\varphi_2$, but for larger weight change rate restriction, our methods has advantages in preserving the safety properties in $\Phi$, espeically for local robustness properties.

\section{Related Works} \label{sec:relatedwork}
In this section, we introduce the topics which are closely related to this work.

\textbf{DNN verification.}
In 2010, Luca Pulita et al. proposed the first method based on partition-refinement for DNN verification in~\cite{DBLP:conf/cav/PulinaT10}.   
Approaches of DNN verification include constraint solving~\cite{LM2017,DBLP:conf/aaai/NarodytskaKRSW18,DBLP:journals/jmlr/BunelLTTKK20,DBLP:conf/cvpr/LinYCZLLH19}, abstract interpretation~\cite{AI2,deeppoly,deepz,deeppolygpu,krelu,deepsymbol}, layer-by-layer exhaustive search~\cite{DBLP:conf/cav/HuangKWW17}, global optimization~\cite{RHK2018,DBLP:conf/nfm/DuttaJST18,DBLP:conf/ijcai/RuanWSHKK19}, functional approximation~\cite{fastlin}, reduction to two-player games~\cite{DBLP:conf/tacas/WickerHK18,DBLP:journals/tcs/WuWRHK20}, star set abstraction~\cite{starset19,imagestar20,DBLP:sas2021},  CEGAR~\cite{abstractioncav20,atva2020,atva2022,tacas2021}, and PAC learning~\cite{DeepPAC,deeppaccite5,deeppaccite6,deeppaccite11,deeppaccite44,deeppaccite74,deeppaccite74,deeppaccite75,deeppaccite78}. For the incremental situation considered in this work, constraint solving, especially SMT, is the most suitable for an incremental design, since there has been a lot of mature technique in incremental SMT solving. Also, weight change of a DNN naturally requires a step of constructing the constraints and initializing the numerical bounds for the modified DNN. For many abstraction-based methods, this step itself is the whole verification procedure, so an incremental design hardly brings them improvement in efficiency. 

\textbf{Incremental constraint solving and sensitivity analysis.}
The incremental constraint solving problem was first proposed in  the 1990s~\cite{DBLP:journals/jlp/Hooker93}, in which it is considered how to solve a  similar SAT problem while reusing the internal state and the information learned of the solution procedure. So far, incremental constraint solving has been applied successfully to bounded model checking~\cite{DBLP:journals/entcs/EenS03,DBLP:conf/charme/Shtrichman01,DBLP:conf/fmics/SchrammelKBMTB15}, and counterexample-based abstraction and proof-based abstraction~\cite{DBLP:conf/fmcad/EenMA10}.
Two main approaches to sensitivity analysis in linear programming are the tolerance approach~\cite{DBLP:journals/networks/RaviW88,https://doi.org/10.1111/j.1540-5915.1991.tb00365.x,FILIPPI20051} and the global approach~\cite{DBLP:journals/eor/BorgonovoBW18,doi:10.1287/opre.43.6.948}. The tolerance approach aims to identify the max variation in the data which remains solution optimal, while the goal of global sensitivity analysis is the identification of the key variability drivers.

\textbf{DNN repair.}
There have been multiple attempts on repairing machine learning models to remove undesirable behaviors. In \cite{GoldbergerKAK20}, Goldberger \emph{et al}. repair the properties by modifying the weights in one layer using one or multiple specific inputs, but its performance are constricted by quality of inputs and the underlying verification solve. NNRepair~\cite{NNrepair} fixes undesirable behaviors using constraint-based repair on neural networks, but only makes modification on a single layer. PRDNN~\cite{prdnn} proposes decoupled DNNs, and reduces the provable repair problem to a linear programming problem. However, the result of this method, PRDNN, leads to the repaired network becoming discontinuous. CARE~\cite{care} is a method which uses causality-based technique to locate the fault and then fixes unexpected properties by modifying the identified weight. In \cite{art}, Lin et al.propose a method which is aimed to train neural networks with provable guarantees. The core idea is to 
abstract both the input space and the network itself, which is used to optimize the neural network to guarantee the property while not compromising too much accuracy.
All these methods does not consider the multi-objective purpose in this work.

\section{Conclusion} \label{sec:conclusion}
In this work, we propose the problem of incremental constraint solving for DNN verification and give an algorithm of incremental SMT solving based on the Reluplex framework. 
We are following the intuition that the $\mathsf{UNSAT}$ inference in the old verification procedure can mostly be inherited, so we restore the most important features  of the last non-terminating configuration on the leaves for incremental SMT solving. 
We implement our algorithm as an incremental SMT solver DeepInc, and the experimental results show that our method is more efficient in most modification cases.
Based on the experimental results, we propose an algorithm for multi-objective DNN repair so that other potential safety properties than the repaired one are preserved after the repair.
By invoking DeepInc, we compare our method with state-of-the-art DNN repair algorithm, and the results show that our method breaks fewer safety properties.

As for future works, we will further consider more patterns of modification, like the changes in the structure of DNNs, to adapt incremental constraint solving to more applications like the CEGAR framework of DNN verification. Another challenging work is to consider incremental SMT solving with a CDCL design. Currently CDCL is not yet maturely considered in SMT-based DNN verification. If CDCL is fused into the current Reluplex framework, then it is worth considering its incremental solving. Other future works may include combination of incremental constraint solving and DNN repair, attack and defense, incremental constraint solving from the perspective of constraint  disturbance for DNN verification, and more heuristic optimizations for the implementation of incremental SMT solving algorithms.

\bibliographystyle{abbrv}
\bibliography{refs}

%%
%% If your work has an appendix, this is the place to put it.
%\appendix

%\section{Research Methods}

%\subsection{Part One}

%\subsection{Part Two}

%\section{Online Resources}

\end{document}